\documentclass[final]{cvpr}

\usepackage{times}
\usepackage{epsfig}
\usepackage{graphicx}
\usepackage[section]{placeins}
\usepackage{amsthm}
\usepackage{amsmath}
\usepackage{thmtools}
\usepackage{amssymb}
\usepackage{comment}
\usepackage{color}
\usepackage{multirow}
\usepackage[linesnumbered,ruled]{algorithm2e}
\usepackage{algorithmicx}
\usepackage[table]{xcolor}
\usepackage{arydshln}
\usepackage{thm-restate}
\usepackage{enumitem}
\usepackage{listings}
\usepackage[misc]{ifsym}
\usepackage{bbding}
\makeatletter
\newcommand{\removelatexerror}{\let\@latex@error\@gobble}
\makeatother

\newtheorem{theorem}{Theorem}
\newtheorem{corollary}{Corollary}[theorem]

\newtheorem{proposition}[theorem]{Proposition}

\usepackage[pagebackref=true,breaklinks=true,colorlinks,bookmarks=false]{hyperref}
\usepackage{nopageno}

\def\cvprPaperID{2257} 
\def\confYear{CVPR 2021}

\begin{document}


\title{clDice - a Novel Topology-Preserving Loss Function for Tubular Structure Segmentation}

\author{Suprosanna Shit \thanks{The authors contributed equally to the work} ~$^{1}$ ~~~~
Johannes C. Paetzold $^\ast$~$^{1}$~~~~
Anjany Sekuboyina$^{1}$~~~~
Ivan Ezhov$^{1}$~~~~\\
Alexander Unger$^{1}$~~~~
Andrey Zhylka$^{2}$~~~~
Josien P. W. Pluim$^{2}$~~~~
Ulrich Bauer$^{1}$~~~~
Bjoern H. Menze$^{1}$~~~~\\
\hfil$^{1}$Technical University of Munich~~~~~$^{2}$ Eindhoven University of Technology\vspace{0.1cm}
}

\maketitle
\begin{abstract} Accurate segmentation of tubular, network-like structures, such as vessels, neurons, or roads, is relevant to many fields of research. For such structures, the topology is their most important characteristic; particularly preserving connectedness: in the case of vascular networks, missing a connected vessel entirely alters the blood-flow dynamics. We introduce a novel similarity measure termed centerlineDice (short \textit{clDice}), which is calculated on the intersection of the segmentation masks and their (morphological) skeleta. We theoretically prove that \textit{clDice} guarantees topology preservation up to homotopy equivalence for binary 2D and 3D segmentation. Extending this, we propose a computationally efficient, differentiable loss function (\textit{soft-clDice}) for training arbitrary neural segmentation networks. We benchmark the \textit{soft-clDice} loss on five public datasets, including vessels, roads and neurons (2D and 3D). Training on \textit{soft-clDice} leads to segmentation with more accurate connectivity information, higher graph similarity, and better volumetric scores.
\end{abstract}

\section{Introduction}
\begin{figure}[ht!]
\label{figmot}
\begin{center}
\includegraphics[width=0.44\textwidth]{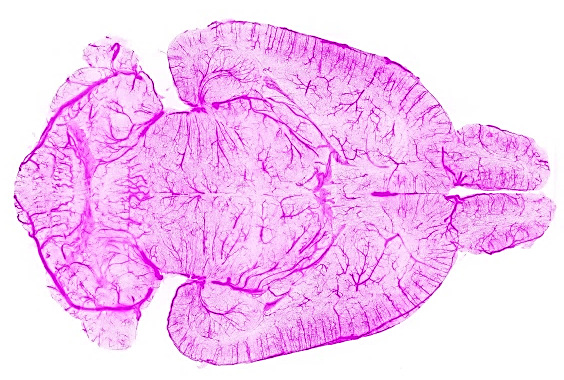}
\includegraphics[width=0.15\textwidth]{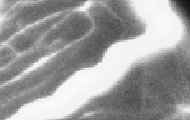}
\includegraphics[width=0.15\textwidth]{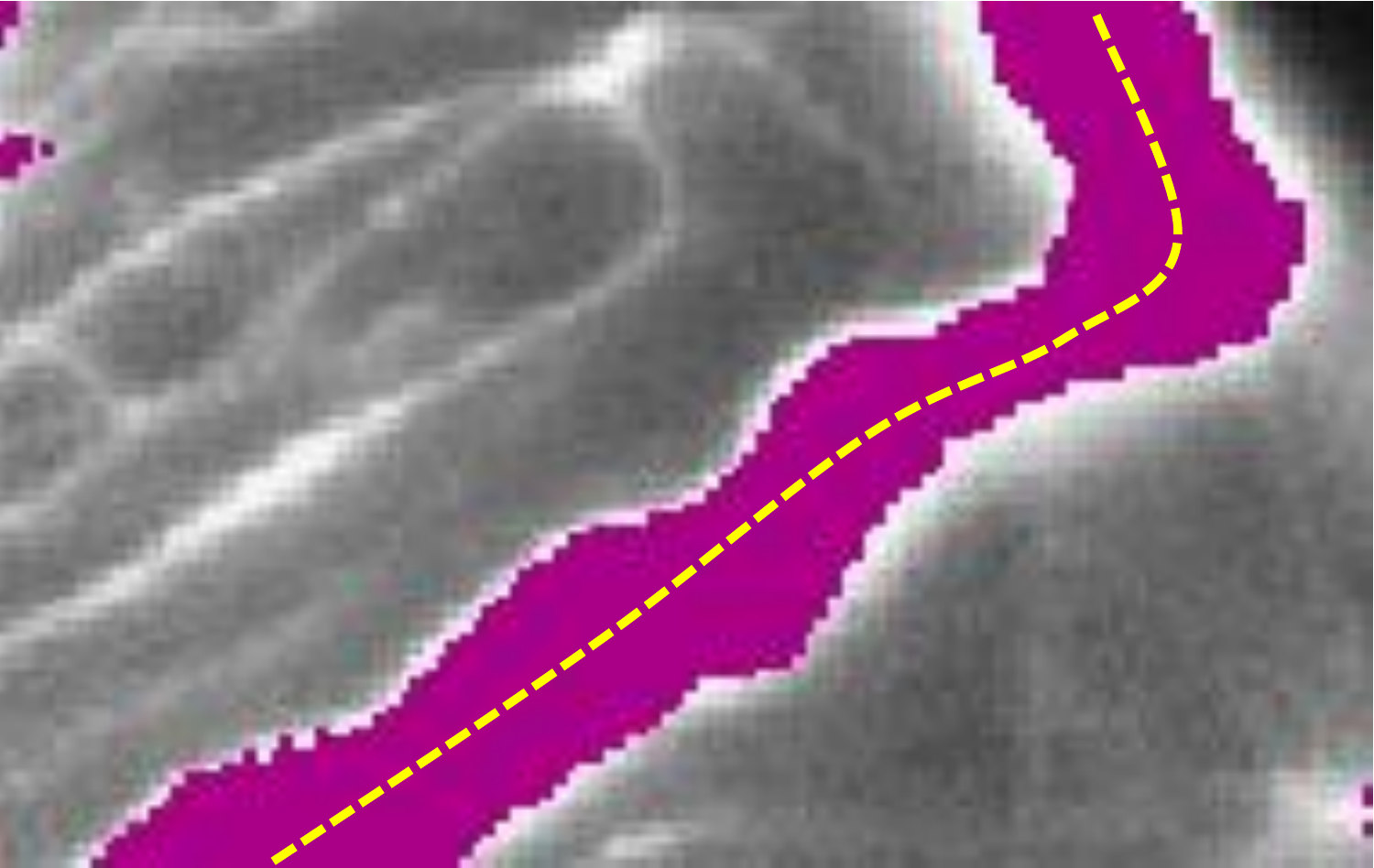}
\includegraphics[width=0.15\textwidth]{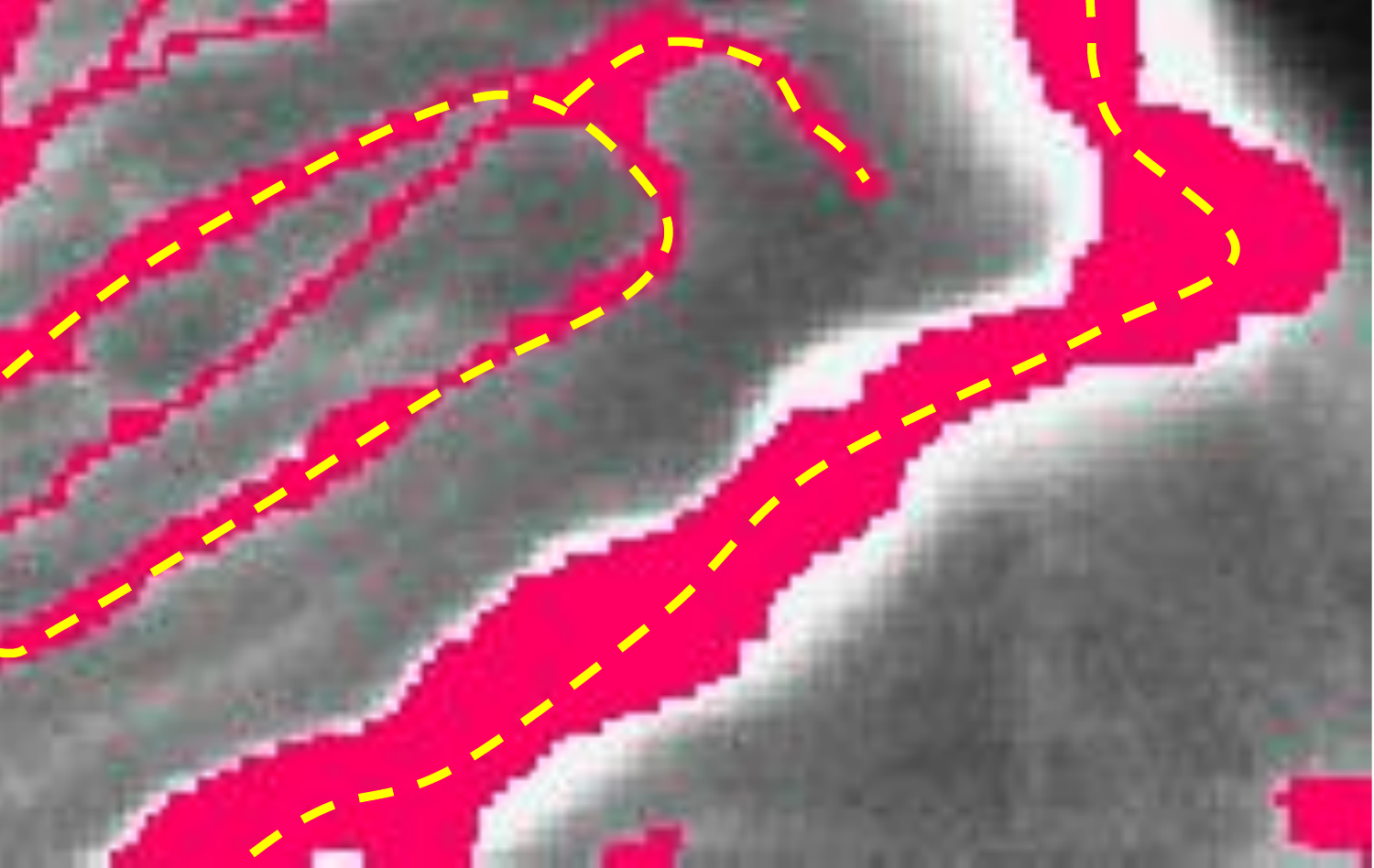}
\end{center}
\caption{\textbf{Motivation:} The figure shows a 3D rendering of a complex, whole brain vascular dataset \cite{todorov2019automated}, where an exemplary 2D slice of the data is chosen and segmented by two different models, see purple (middle) and red (right), respectively. The two segmentation results achieve identical quality in terms of the traditional Dice score. Note that the purple segmentation does not capture the small vessels while segmenting the large vessel very accurately; on the other side, the red segmentation captures all vessels in the image while being less accurate on the radius of the large vessel. Skeleta are drawn in yellow. From a topology or network perspective, the red segmentation is evidently preferred.}
\vspace{-1.5em}
\end{figure}
Segmentation of \textit{tubular} and \textit{curvilinear} structures is an essential problem in numerous domains, such as clinical and biological applications (blood vessel and neuron segmentation from microscopic, optoacoustic, or radiology images), remote sensing applications (road network segmentation from satellite images) and industrial quality control, etc. In the aforementioned domains, a topologically accurate segmentation is necessary to guarantee error-free down-stream tasks, such as computational hemodynamics, route planning, Alzheimer's disease prediction \cite{hunter2012morphological}, or stroke modeling \cite{joutel2010cerebrovascular}. When optimizing computational algorithms for segmenting curvilinear structures, the two most commonly used categories of quantitative performance measures for evaluating segmentation accuracy of \textit{tubular} structures, are 1) overlap based measures such as Dice, precision, recall, and Jaccard index; and 2) volumetric distance measures such as the Hausdorff and Mahalanobis distance \cite{kirbas2004review,schneider2015joint,phellan2017vascular,hu2018retinal}.

However, in most segmentation problems, where the object of interest is 1) locally a \textit{tubular} structure and 2) globally forms a \textit{network}, the most important characteristic is the connectivity of the global network topology. Note that \textit{network} in this context implies a physically connected structure, such as a vessel network, a road network, etc., which is also the primary structure of interest for the given image data. As an example, one can refer to brain vasculature analysis, where a missed vessel segment in the segmentation mask can pathologically be interpreted as a stroke or may lead to dramatic changes in a global simulation of blood flow. On the other hand, limited over- or under-segmentation of vessel radius can be tolerated, because it does not affect clinical diagnosis.

For evaluating segmentations in such tubular-network structures, traditional volume-based performance indices are sub-optimal. For example, Dice and Jaccard rely on the average voxel-wise hit or miss prediction \cite{taha2015metrics}. In a task like network-topology extraction, a spatially contiguous sequence of correct voxel prediction is more meaningful than a spurious correct prediction. This ambiguity is relevant for objects of interest, which are of the same thickness as the resolution of the signal. For them, it is evident that a single-voxel shift in the prediction can change the topology of the whole network. Further, a globally averaged metric does not equally weight tubular-structures with large, medium, and small radii (cf. Fig~\ref{figmot}). In real vessel datasets, where vessels of wide radius ranges exist, e.g. 30 $\mu$m for arterioles and 5 $\mu$m for capillaries \cite{todorov2019automated,di2018whole}, training on a globally averaged loss induces a strong bias towards the volumetric segmentation of large vessels. Both scenarios are pronounced in imaging modalities, such as fluorescence microscopy \cite{todorov2019automated,zhao2020cellular} and optoacoustics, which focus on mapping small capillary structures. 

To this end, we are interested in a topology-aware image segmentation, eventually enabling a correct network extraction. Therefore, we ask the following research questions: 
\begin{enumerate}
    \item[Q1.] What is a good pixelwise measure to benchmark segmentation algorithms for \textbf{tubular}, and related linear and curvilinear structure segmentation while guaranteeing the preservation of the \textbf{network-topology}? 
    \item[Q2.] Can we use this \textit{improved measure} as a loss function for neural networks, which is a void in existing literature?
\end{enumerate}
\subsection{Related Literature}
Achieving topology preservation can be crucial to obtain meaningful segmentation, particularly for elongated and connected shapes, e.g. vascular structures or roads. However, analyzing preservation of topology while simplifying geometries is a difficult analytical and computational problem \cite{edelsbrunner2010computational,edelsbrunner2000topological}.

For binary geometries, various algorithms based on thinning and medial surfaces have been proven to be topology-preserving according to varying definitions of topology \cite{kong1995topology,lee1994building,ma1994topology,palagyi20023}. For non-binary geometries, existing methods applied topology and connectivity constraints onto variational and Markov random field-based methods: tree shape priors for vessel segmentation \cite{stuhmer2013tree}, graph representation priors to natural images \cite{andres2011probabilistic}, higher-order cliques which connect superpixels \cite{wegner2013higher} and adversarial learning for road segmentation \cite{vasu2020topoal}, integer programming to general curvilinear structures \cite{turetken2016reconstructing}, and proposed a tree-structured convolutional gated recurrent unit \cite{kong2020learning}, morphological optimization \cite{gur2019unsupervised}, among others \cite{araujo2019deep,han2003topology,nowozin2009global,navarro2019shape,oswald2014generalized,rempfler2017efficient,segonne2008active,vicente2008graph,zeng2008topology,wu2016deep}. Further, topological priors of containment were applied to histology scans \cite{bentaieb2016topology}, a 3D CNN with graph refinement was used to improve airway connectivity \cite{jin20173d}, and recently, Mosinska et al. trained networks which perform segmentation and path classification simultaneously \cite{mosinska2019joint}. Another approach enables the predefinition of Betti numbers and enforces them on the training\cite{clough2020topological}.

The aforementioned literature has advanced the communities understanding of topology-preservation, but critically, they do not possess end-to-end loss functions that optimize topology-preservation. 
In this context, the literature remains sparse. Recently, 
Mosinska et al. suggested that pixel-wise loss-functions are unsuitable and used selected filter responses from a VGG19 network as an additional penalty \cite{mosinska2018beyond}. Nonetheless, their approach does not prove topology preservation. Importantly, Hu et al. proposed the first continuous-valued loss function based on the Betti number and persistent homology \cite{hu2019topology}. However, this method is based on matching critical points, which, according to the authors makes the training very expensive and error-prone for real image-sized patches \cite{hu2019topology}. While this is already limiting for a translation to large real world data set, we find that none of these approaches have been extended to three dimensional (3D) data.

\subsection{Our Contributions}
The objective of this paper is to identify an efficient, general, and intuitive loss function that enables topology preservation while segmenting tubular objects. 
We introduce a novel connectivity-aware similarity measure named \textit{clDice} for benchmarking tubular-segmentation algorithms. 
Importantly, we provide theoretical guarantees for the topological correctness of the \textit{clDice} for binary 2D and 3D segmentation. As a consequence of its formulation based on morphological skeletons, our measure pronounces the network's topology instead of equally weighting every voxel.
Using a differentiable soft-skeletonization, we show that the \textit{clDice} measure can be used to train neural networks. 
We show experimental results for various 2D and 3D network segmentation tasks to demonstrate the practical applicability of our proposed similarity measure and loss function.

\section{Let's Emphasize \emph{Connectivity}}
We propose a novel connectivity-preserving metric to evaluate tubular and linear structure segmentation based on intersecting skeletons with masks. We call this metric \textit{centerlineDice} or \textbf{\textit{clDice}}.
\begin{figure*}[]
\begin{center}
\includegraphics[width=0.85\textwidth]{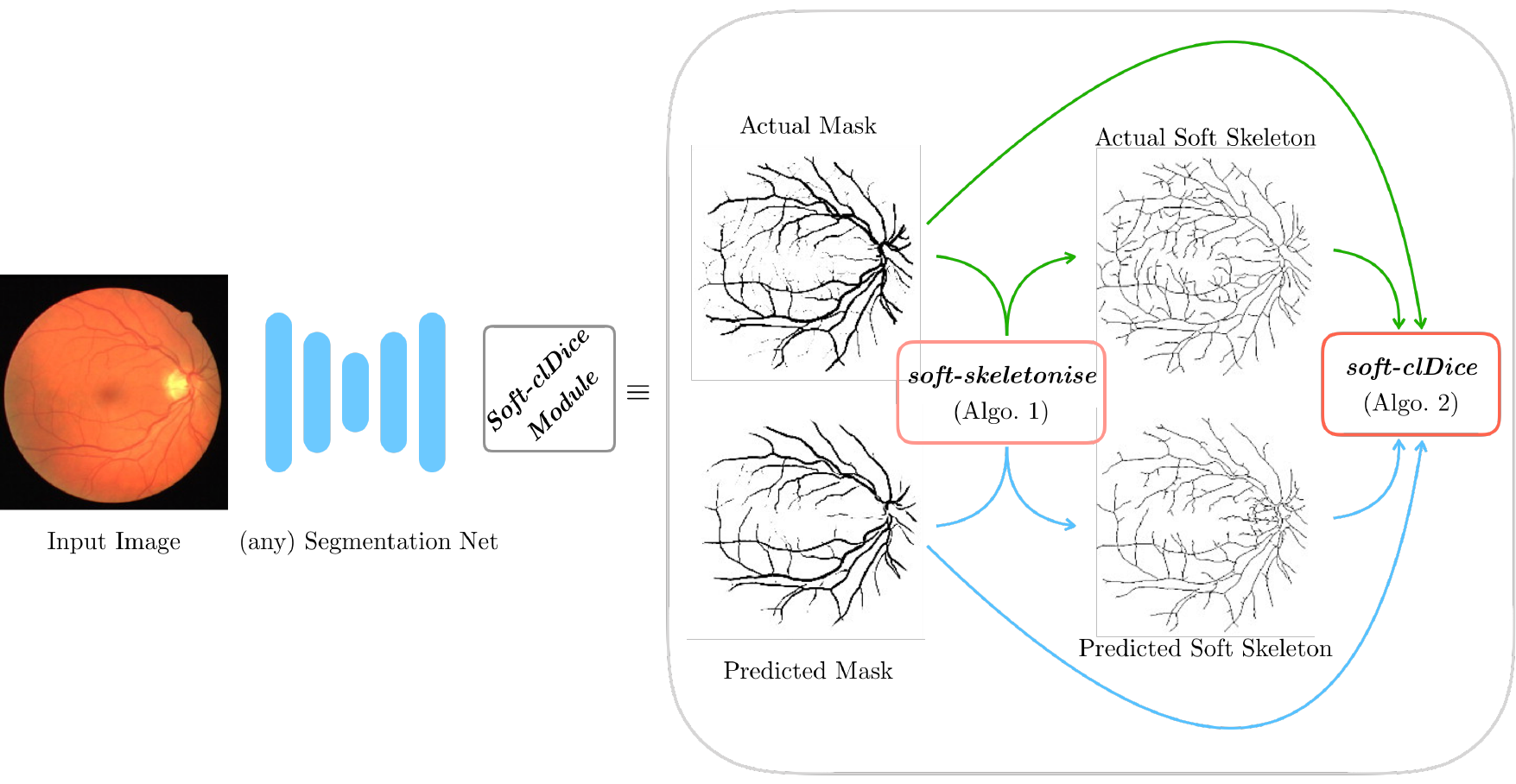}
\end{center}
\label{met}
\caption{ \textbf{Schematic overview of our proposed method:} Our proposed \textit{clDice} loss can be applied to any arbitrary segmentation network. The soft-skeletonization can be easily implemented using pooling functions from any standard deep-learning toolbox.}
\vspace{-1em}
\end{figure*}

We consider two binary masks: the ground truth mask ($V_L$) and the predicted segmentation masks ($V_P$). First, the skeletons $ S_P$ and $ S_L$ are extracted from $V_P$ and $V_L$ respectively. Subsequently, we compute the fraction of $S_P$ that lies within $V_L$, which we call \textit{Topology Precision} or $\operatorname{Tprec}(S_P, V_L)$, and vice-a-versa we obtain \textit{Topology Sensitivity} or $\operatorname{Tsens}(S_L, V_P)$ as defined bellow;
\begin{align}
\operatorname{Tprec}(S_P, V_L) = \frac{|S_P  \cap V_L|}{|S_P|};~~
\operatorname{Tsens}(S_L, V_P) = \frac{|S_L  \cap V_P|}{|S_L|}\label{top_def}
\end{align}
We observe that the measure $\operatorname{Tprec}(S_P, V_L)$ is susceptible to false positives in the prediction while the measure $\operatorname{Tsens}(S_L, V_P)$ is susceptible to false negatives. This explains our rationale behind referring to the $\operatorname{Tprec}(S_P, V_L)$ as topology's precision and to the $\operatorname{Tsens}(S_L, V_P)$ as its sensitivity. Since we want to maximize both precision and sensitivity (recall), we define \textit{clDice} to be the harmonic mean (also known as F1 or Dice) of both the measures:
\begin{align}
\operatorname{clDice}(V_P, V_L) & = 2 \times \dfrac{ \operatorname{Tprec}(S_P, V_L) \times \operatorname{Tsens}(S_L, V_P)}{\operatorname{Tprec}(S_P, V_L) + \operatorname{Tsens}(S_L, V_P)}\label{eq2}
\end{align}
Note that our \textit{clDice} formulation is not defined for $\operatorname{Tprec} = 0 \mbox{ and } \operatorname{Tsens} = 0$, but can easily be extended continuously with the value $0$.
\section{Topological Guarantees for clDice}
\label{sec:proof}

The following section provides general theoretical guarantees for the preservation of topological properties achieved by optimizing \textit{clDice} under mild conditions on the input.
Roughly, these conditions state that the object of interest is embedded in $S^3$ in a non-knotted way, as is typically the case for blood vessel and road structures.

Specifically, we assume that both ground truth and prediction \emph{admit foreground and background skeleta}, which means that both foreground and background are homotopy-equivalent to topological graphs, which we assume to be embedded as \emph{skeleta}.
Here, the voxel grid is considered as a cubical complex, consisting of elementary cubes of dimensions 0, 1, 2, and 3.
This is a special case of a \emph{cell complex} (specifically, a \emph{CW complex}), which is a space constructed inductively, starting with isolated points ($0$-cells), and gluing a collection of topological balls of dimension $k$ (called \emph{$k$-cells}) along their boundary spheres to a $k-1$-dimensional complex.
The voxel grid, seen as a cell complex in this sense,
can be completed to an ambient complex that is homeomorphic to the 3-sphere $S^3$ by attaching a single exterior cell to the boundary.
In order to consider foreground and background of a binary image as complementary subspaces,
the foreground is now assumed to be the union of closed unit cubes in the voxel grid, corresponding to voxels with value $1$; and the background is the complement in the ambient complex.
This convention is commonly used in digital topology \cite{kong1989digital,kong1995topology}.
The assumption on the background can then be replaced by a convenient equivalent condition, stating that the foreground is also homotopy equivalent to a subcomplex obtained from the ambient complex by only removing 3-cells and 2-cells.
Such a subcomplex is then clearly homotopy-equivalent to the complement of a 1-complex.\\



We will now observe that the above assumptions imply that the foreground and the background are connected and have a free fundamental group and vanishing higher fundamental groups.
In particular, the homotopy type is already determined by the first Betti number \footnote{Betti numbers:
	$\beta_0$ represents the number of distinct \textit{connected-components},
	$\beta_1$ represents the number of \textit{circular holes}
	, and
	$\beta_2$ represents the number of \textit{cavities}, for depictions see Supplementary material }; moreover, a map inducing an isomorphism in homology is already a homotopy equivalence.
To see this, first note that both foreground and background are assumed to have the homology of a graph, in particular, homology is trivial in degree 2.
By Alexander duality \cite{aleksandrov1998combinatorial}, then, both foreground and background have trivial reduced cohomology in degree 0, meaning that they are connected.
This implies that both have a free fundamental group (as any connected graph) and vanishing higher homotopy groups.
In particular, since homology in degree 1 is the Abelianization of the fundamental group, these two groups are isomorphic.
This in turn implies that in our setting a map that induces isomorphisms in homology already induces isomorphisms between all homotopy groups.
By Whitehead's theorem \cite{whitehead1949combinatorial}, such a map is then a homotopy equivalence.\\

The following theorem shows that under our assumptions on the images admitting foreground and background skeleta, the existence of certain nested inclusions already implies the homotopy-equivalence of foreground and background, which we refer to as \emph{topology preservation}.

\begin{theorem}
	\label{thm2}
	Let $L_A \subseteq A \subseteq K_A$ 
	and $L_B \subseteq B \subseteq K_B$
	be connected subcomplexes of some cell complex.
	Assume that the above inclusions
	are homotopy equivalences.
	If the subcomplexes also are related by inclusions
	$L_A \subseteq B \subseteq K_A$ 
	and $L_B \subseteq A \subseteq K_B$, then these inclusions must be homotopy equivalences as well.
	In particular, $A$ and $B$ are homotopy-equivalent.
\end{theorem}

\begin{proof}
	An inclusion of connected cell complexes is a homotopy equivalence if and only if it induces isomorphisms on all homotopy groups.
	Since the inclusion $L_A \subseteq B \subseteq K_A$ induces an isomorphism, the inclusion $L_A \subseteq B$ induces a monomorphism, and since $B \subseteq K_B$ induces an isomorphism, the inclusion $L_A \subseteq K_B$ also induces a monomorphism.
	At the same time, since the inclusion $L_B \subseteq A \subseteq K_B$ induces an isomorphism, the inclusion $A \subseteq K_B$ induces an epiorphism, and since $L_A \subseteq A$ induces an isomorphism, the inclusion $L_A \subseteq K_B$ also induces an epiorphism.
	Together, this implies that the inclusion $L_A \subseteq K_B$ induces an isomorphism.
	
	Together with the isomorphisms induced by $L_A \subseteq A$ and $B \subseteq K_B$, we obtain isomorphisms induced by $L_A \subseteq B$ and by $A \subseteq K_B$, which compose to an isomorphism
	between the homotopy groups of $A$ and $B$.
\end{proof}

\begin{corollary}
	Let $V_L$ and $V_P$ be two binary masks admitting foreground and background skeleta, such that the foreground skeleton of $V_L$ is included in the foreground of $V_P$ and vice versa, and similarly for the background.
	Then the foregrounds of $V_L$ and $V_P$ are homotopy equivalent, and the same is true for their backgrounds.
\end{corollary}

Note that the inclusion condition in this corollary is satisfied if and only if \textit{clDice} evaluates to $1$ on both foreground and background of $(V_L,V_P)$.

This proof lays the ground for a general interpretation of \textit{clDice} as a topology preserving metric. Additionally, we provide an elaborate explanation of \textit{clDice} topological properties, using concepts of applied digital topology in the theory section of the Supplementary material \cite{kong1989digital,kong1995topology}.
\section{Training Neural Networks with \textbf{\textit{clDice}}}
In the previous section we provided general theoretic guarantees how \textit{clDice} has topology preserving properties. The following chapter shows how we applied our theory to efficiently train topology preserving networks using the \textit{clDice} formulation. \footnote{\url{https://github.com/jocpae/clDice}}

\begin{figure*}[]
\begin{center}
\includegraphics[width=0.9\linewidth]{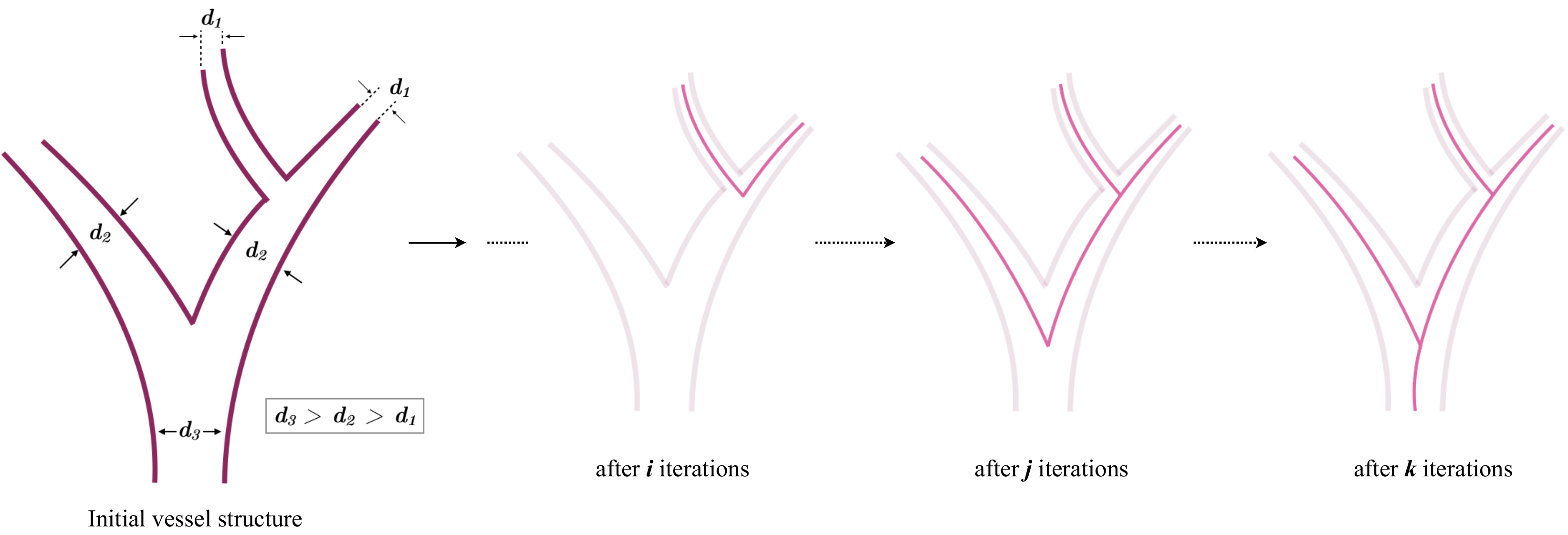}
\end{center}
\caption{ Based on the initial vessel structure (purple), sequential bagging of skeleton voxels (red) via iterative skeletonization leads to a complete skeletonization, where $d$ denotes the diameter and  $k>j>i$ iterations.}
\label{seq_skel}
\vspace{-1em}
\end{figure*}

\subsection{\textbf{\textit{Soft-clDice}} using \textbf{\textit{Soft-skeletonization}}:}
Extracting accurate skeletons is essential to our method. For this task, a multitude of approaches has been proposed. However, most of them are not fully differentiable and therefore unsuited to be used in a loss function. Popular approaches use the Euclidean distance transform or utilize repeated morphological thinning. Euclidean distance transform has been used on multiple occasions \cite{shih1995skeletonization,wright1995skeletonization}, but remains a discrete operation and, to the best of our knowledge, an end-to-end differentiable approximation remains to be developed, preventing the use in a loss function for training neural networks.
On the contrary, morphological thinning is a sequence of dilation and erosion operations [c.f. Fig. \ref{seq_skel}].

\begin{figure}[t!]
\begin{minipage}{0.45\textwidth}
\removelatexerror
    \begin{algorithm*}[H]
        \caption{\textit{soft-skeleton}}
        \label{algorithm-1}
        \SetKwInOut{Input}{Input}
        \Input{$I,k$}
        \begin{algorithmic}
        \State $I' \gets \textit{maxpool}(\textit{minpool}(I))$
        \State $S \gets \textit{ReLU}(I-I')$
        \end{algorithmic}
        \For{$i\gets0$ \KwTo $k$}{
            $I \gets  \textit{minpool}(I)$\\
            $I' \gets \textit{maxpool}(\textit{minpool}(I))$\\
            $S \gets S+(1-S)\circ \textit{ReLU}(I-I')$
        }
        \SetKwInOut{Output}{Output}
        \Output{$S$}
    \end{algorithm*}
    \begin{algorithm*}[H]
        \caption{\textit{soft-clDice}}
        \label{algorithm-2}
        \SetKwInOut{Input}{Input}
        \Input{$V_P,V_L$}
        \begin{algorithmic}
            \State $S_P \gets \mbox{\textit{soft-skeleton}}(V_P)$
            \State $S_L \gets \mbox{\textit{soft-skeleton}}(V_L)$
            \State $\textit{Tprec}(S_P, V_L) \gets \frac{|S_P \circ V_L|+\epsilon}{|S_P|+\epsilon}$
            \State $\textit{Tsens}(S_L, V_P) \gets \frac{|S_L \circ V_P|+\epsilon}{|S_L|+\epsilon}$
            \State $\textit{clDice} \gets$ \par
            $2 \times \frac{ \textit{Tprec}(S_P, V_L)\times \textit{Tsens}(S_L, V_P)}{\textit{Tprec}(S_P, V_L)+ \textit{Tsens}(S_L, V_P)}$
        \end{algorithmic}
        \vspace{0.2cm}
        \SetKwInOut{Output}{Output}
        \Output{$clDice$}
    \end{algorithm*}
\end{minipage}    
 \caption{ \textbf{Algorithm \ref{algorithm-1}} calculates the proposed \textit{soft-skeleton}, here $I$ is the mask to be \textit{soft-skeletonized} and $k$ is the number of iterations for skeletonization. \textbf{Algorithm \ref{algorithm-2}}, calculates the \textit{soft-clDice} loss, where $V_P$ is a real-valued probabilistic prediction from a segmentation network and $V_L$ is the true mask. We denote Hadamard product using $\circ$.}
 \vspace{-1.5em}
\end{figure}
Importantly, thinning using morphological operations (skeletonization) on curvilinear structures can be topology-preserving \cite{palagyi20023}. Min- and max filters are commonly used as the grey-scale alternative of morphological dilation and erosion. Motivated by this, we propose `soft-skeletonization', where an iterative min- and max-pooling is applied as a proxy for morphological erosion and dilation. The Algorithm \ref{algorithm-1} describes the iterative processes involved in its computation. The hyper-parameter $k$ involved in its computation represents the iterations and has to be greater than or equal to the maximum observed radius. In our experiments, this parameter depends on the dataset. For example, it is  $k=5...25$ in our experiments, matching the pixel radius of the largest observed tubular structures. Choosing a larger $k$ does not reduce performance but increases computation time. On the other hand, a too low $k$ leads to incomplete skeletonization. 


In Figure \ref{seq_skel}, the successive steps of our skeletonization are intuitively represented. In the early iterations, the structures with a small radius are skeletonized and preserved until the later iterations when the thicker structures become skeletonized. This enables the extraction of a parameter-free, morphologically motivated soft-skeleton. The aforementioned soft-skeletonization enables us to use \textit{clDice} as a fully differentiable, real-valued, optimizable measure. The Algorithm \ref{algorithm-2} describes its implementation. We refer to this as the \textit{soft-clDice}.

For a single connected foreground component and in the absence of knots, the homotopy type is specified by the number of linked loops. Hence, if the reference and the predicted volumes are not homotopy equivalent, they do not have pairwise linked loops. To include these missing loops or exclude the extra loops, one has to add or discard deformation retracted skeleta of the solid foreground. This implies adding \textit{new correctly predicted voxels}. In contrast to other volumetric losses such as Dice, cross-entropy, etc., \textit{clDice} only considers the deformation-retracted graphs of the solid foreground structure. Thus, we claim that \textit{clDice} requires the least amount of \textit{new correctly predicted voxels} to guarantee the homotopy equivalence. Along these lines, Dice or cross-entropy can only guarantee homotopy equivalence if every single voxel is segmented correctly. On the other hand, \textit{clDice} can guarantee homotopy equivalence for a broader combinations of connected-voxels. Intuitively, this is a very much desirable property as it makes \textit{clDice} robust towards outliers and noisy segmentation labels.

\subsection{Cost Function}
Since our objective here is to preserve topology while achieving accurate segmentations, and not to learn skeleta, we combine our proposed \textit{soft-clDice} with \textit{soft-Dice} in the following manner: 
\begin{equation}
\mathcal{L}_{c} = (1-\alpha)(1-\mbox{\textit{soft\textbf{Dice}})}+\alpha(1-\mbox{\textit{soft\textbf{clDice}}) }
\label{eq3}
\end{equation}
where $\alpha \in [0,0.5]$. In stark contrast to previous works, where segmentation and centerline prediction has been learned jointly as multi-task learning \cite{uslu2018multi,tetteh2018deepvesselnet}, we are not interested in learning the centerline. We are interested in learning a topology-preserving segmentation. Therefore, we restrict our experimental choice of alpha to $\alpha \in [0,0.5]$.
We test \textit{clDice} on two state-of-the-art network architectures: i) a 2D and 3D U-Net\cite{ronneberger2015u,cciccek20163d}, and ii) a 2D and 3D fully connected networks (FCN) \cite{tetteh2018deepvesselnet,gerlRSOM}. As baselines, we use the same architectures trained using \textit{soft-Dice} \cite{milletari2016v,sudre2017generalised}.

\subsection{Adaption for Highly Imbalanced Data}
Our theory  (Section \ref{sec:proof}), describes a two-class problem where \textit{clDice} should be computed on both the foreground and the background channels. In our experiments, we show that for complex and highly imbalanced dataset it is sufficient to calculate the \textbf{clDice} loss on the underrepresented foreground class. We attribute this to the distinct properties of tubularness, sparsity of foreground and the lack of cavities (Betti number 2) in our data. An intuitive interpretation how these assumptions are valid in terms of digital topology can be found in the supplementary material.


\section{Experiments}
\subsection{Datasets} 
We employ five public datasets for validating \textit{clDice} and \textit{soft-clDice} as a measure and an objective function, respectively. In 2D, we evaluate on the DRIVE retina dataset \cite{staal2004ridge}, the Massachusetts Roads dataset \cite{MnihThesis} and the CREMI neuron dataset \cite{Funke_2019}. In 3D, a synthetic vessel dataset with an added Gaussian noise term \cite{schneider2012tissue} and the Vessap dataset of multi-channel volumetric scans of brain vessels is used \cite{todorov2019automated,paetzold2019transfer}. For the Vessap dataset we train different models for one and two input channels. 
%
For all of the datasets, we perform three fold cross-validation and test on held-out, large, and highly-variant test sets. Details concerning the experimental setup can be found in the supplementary.
\subsection{Evaluation Metrics}
We compare the performance of various experimental setups using three types of metrics: volumetric, topology-based, and graph-based. 
\begin{enumerate}[itemsep=0pt,parsep=0pt]
    \item Volumetric: We compute volumetric scores such as Dice coefficient, Accuracy, and the proposed \textit{clDice}.
    \item Topology-based: We calculate the mean of absolute Betti Errors for the Betti Numbers $\beta_0$ and $\beta_1$ and the mean absolute error of Euler characteristic, $\chi = V-E+F$, where $V, E, \mbox{ and } F$ denotes number of vertices, edges, and faces.
    \item Graph-based: we extract random patch-wise graphs for the 2D/3D images. We uniformly sample fixed number of points from the graph and compute the StreetmoverDistance (SMD) \cite{belli2019image}. SMD captures a Wasserstein distance between two graphs. Additionally we compute the F1 score of junction-based metric \cite{citrarotowards}.
\end{enumerate}

\begin{table*}[!ht]
\caption{ Quantitative experimental results for the Massachusetts road dataset (Roads), the CREMI dataset, the DRIVE retina dataset and the Vessap dataset (3D). Bold numbers indicate the best performance. The performance according to the \textit{clDice} measure is highlighted in rose. For all experiments we observe that using \textit{soft-clDice} in $\mathcal{L}_{c}$ results in improved scores compared to \emph{soft-Dice}. This improvement holds for almost $\alpha > 0$; $\alpha$ can be interpreted as a dataset specific hyper-parameter.}

\centering
\label{final_table}
\footnotesize
\begin{tabular}{lll|c c>{\columncolor{red!20}} c|cc|ccc}

\hline\hline
Dataset & Network & Loss & Dice & Accuracy & \textit{clDice} & $\beta_0$ Error &  $\beta_1$ Error  & SMD \cite{belli2019image} &  $\chi_{error}$ & Opt-J F1 \cite{citrarotowards}\\
\hline\hline
\multirow{13}{*}{Roads}    & \multirow{2}{*}{FCN} &  \textit{soft-dice} & 64.84 & 95.16 & 70.79 & 1.474 & 1.408 & 0.1216 & 2.634 & 0.766\\ \cdashline{3-11}[2pt/2pt]
&   & $\mathcal{L}_{c}, \alpha = 0.1$     & 66.52 & 95.70 & 74.80 & 0.987 & 1.227 & 0.1002 & 2.625 & 0.768\\
&   & $\mathcal{L}_{c}, \alpha = 0.2$     & \textbf{67.42}& \textbf{95.80} & 76.25 & \textbf{0.920} & 1.280 & \textbf{0.0954} & 2.526 & 0.770\\
&   & $\mathcal{L}_{c}, \alpha = 0.3$     & 65.90 & 95.35 & 74.86 & 0.974 & 1.197 & 0.1003 & 2.448 & 0.775\\
&   & $\mathcal{L}_{c}, \alpha = 0.4$     & 67.18 & 95.46 & \textbf{76.92} & 0.934 & \textbf{1.092} & 0.0991 & \textbf{2.183} & \textbf{0.803}\\
&   & $\mathcal{L}_{c}, \alpha = 0.5$     & 65.77 & 95.09 & 75.22 & 0.947 & 1.184 & 0.0991 & 2.361 & 0.782\\
  \cline{2-11}

& \multirow{6}{*}{U-NET} & \textit{soft-dice} & 76.23 & 96.75 & 86.83 & 0.491 & 1.256 & 0.0589 & 1.120 & 0.881\\ \cdashline{3-11}[2pt/2pt]
&  & $\mathcal{L}_{c}, \alpha = 0.1$  & \textbf{76.66} & \textbf{96.77} & 87.35 & 0.359 & \textbf{0.938} & 0.0457 & 0.980 & 0.878\\
&  & $\mathcal{L}_{c}, \alpha = 0.2$  & 76.25 & 96.76 & 87.29 & \textbf{0.312} & 1.031 & \textbf{0.0415} & 0.865 & 0.900\\
&  & $\mathcal{L}_{c}, \alpha = 0.3$  & 74.85 & 96.57 & 86.10 & 0.322 & 1.062 & 0.0504 & 0.827 & 0.913\\
&  & $\mathcal{L}_{c}, \alpha = 0.4$  & 75.38 & 96.60 & 86.16 & 0.344 & 1.016 & 0.0483 & \textbf{0.755} & \textbf{0.916}\\
&  & $\mathcal{L}_{c}, \alpha = 0.5$  & 76.45 & 96.64 & \textbf{88.17} & 0.375 & 0.953 & 0.0527 & 1.080 & 0.894\\\cline{2-11}
& Mosinska et al. & \cite{mosinska2018beyond,hu2019topology} & - & 97.54 & - & - & 2.781 & - & - & -\\
& Hu et al. & \cite{hu2019topology} & - & 97.28 & - & - & 1.275 & - & - & -\\
\hline\hline
\multirow{9}{*}{CREMI}   
& \multirow{6}{*}{U-NET} & \textit{soft-dice} & 91.54 & 97.11 & 95.86 & 0.259 & 0.657 & 0.0461 & 1.087 & 0.904\\ \cdashline{3-11}[2pt/2pt]
&  & $\mathcal{L}_{c}, \alpha = 0.1$  & 91.76 & \textbf{97.21} & 96.05 & 0.222 & 0.556 & \textbf{0.0395} & 1.000 & 0.900\\
&  & $\mathcal{L}_{c}, \alpha = 0.2$  & 91.66 & 97.15 & 96.01 & 0.231 & 0.630 & 0.0419 & 0.991 & 0.902\\
&  & $\mathcal{L}_{c}, \alpha = 0.3$  & \textbf{91.78} & 97.18 & \textbf{96.21} & \textbf{0.204} & \textbf{0.537} & 0.0437 & \textbf{0.919} & \textbf{0.913}\\
&  & $\mathcal{L}_{c}, \alpha = 0.4$  & 91.56 & 97.12 & 96.09 & 0.250 & 0.630 & 0.0444 & 0.995 & 0.902\\
&  & $\mathcal{L}_{c}, \alpha = 0.5$  & 91.66 & 97.16 & 96.16 & 0.231 & 0.620 & 0.0455 & 0.991 & 0.907\\\cline{2-11}
& Mosinska et al. & \cite{mosinska2018beyond,hu2019topology} & 82.30 & 94.67 & - & - & 1.973 & - & - & -\\
& Hu et al. & \cite{hu2019topology} & - & 94.56 & - & - & 1.113 & - & - & -\\
\hline\hline                              
\multirow{10}{*}{DRIVE retina~}& \multirow{6}{*}{FCN} & \textit{soft-Dice} & 78.23  & 96.27 & 78.02  & 2.187 & 1.860 & 0.0429 & 3.275 & 0.773\\ \cdashline{3-11}[2pt/2pt]
&   & $\mathcal{L}_{c}, \alpha = 0.1$     & 78.36 & 96.25 & 79.02 & 2.100 & 1.610 & 0.0393 & 3.203 & 0.777\\
&   & $\mathcal{L}_{c}, \alpha = 0.2$     & \textbf{78.75} & 96.29 & 80.22 & 1.892 & 1.382 & 0.0383 & 2.895 & 0.793\\
&   & $\mathcal{L}_{c}, \alpha = 0.3$     & 78.29 & 96.20 & 80.28 & 1.888 & \textbf{1.332} & \textbf{0.0318} & 2.918 & \textbf{0.798}\\
&   & $\mathcal{L}_{c}, \alpha = 0.4$     & 78.00 & 96.11 & 80.43 & 2.036 & 1.602 & 0.0423 & 3.141 & 0.764\\
&   & $\mathcal{L}_{c}, \alpha = 0.5$     & 77.76 & 96.04 & \textbf{80.95} & \textbf{1.836} & 1.408 & 0.0394 & \textbf{2.848} & 0.794\\\cline{2-11}
 & \multirow{2}{*}{U-Net} & \textit{soft-Dice} & 74.25 & 95.63 & 75.71 & 1.745 & 1.455 & 0.0649 & 2.997 & 0.760\\
&    & $\mathcal{L}_{c}, \alpha = 0.5$   & \textbf{75.21} & \textbf{95.82} & \textbf{76.86} & \textbf{1.538} & \textbf{1.389} & \textbf{0.0586} & \textbf{2.737} & \textbf{0.767}\\\cline{2-11}
& Mosinska et al. & \cite{mosinska2018beyond,hu2019topology} & - & 95.43 & - & - & 2.784 & - & - & -\\
& Hu et al. & \cite{hu2019topology} & - & 95.21 & - & - & 1.076 & - & - & -\\
\hline\hline

\multirow{16}{*}{Vessap data}    & \multirow{2}{*}{FCN, 1 ch} &  \textit{soft-dice}  & 85.21 & \textbf{96.03} & 90.88 & 3.385 & 4.458 & 0.00459 & 5.850 & 0.862\\ 
  & &$\mathcal{L}_{c}, \alpha = 0.5$  & \textbf{85.44} & 95.91 & \textbf{91.32} & \textbf{2.292} & \textbf{3.677} & \textbf{0.00417} & \textbf{5.620} & \textbf{0.864}\\
  \cline{2-11}
  
& \multirow{6}{*}{FCN, 2 ch} & \textit{soft-dice} & 85.31 & 95.82 & 90.10 & 2.833 & 4.771 & 0.00629 & 6.080 & 0.849\\\cdashline{3-11}[2pt/2pt]
&  & $\mathcal{L}_{c}, \alpha = 0.1$  & 85.96 & 95.99 & 91.02 & 2.896 & \textbf{4.156} & 0.00447 & 5.980 & 0.860\\
&  & $\mathcal{L}_{c}, \alpha = 0.2$  & \textbf{86.45} & \textbf{96.11} & 91.22 & 2.656 & 4.385 & 0.00466 & 5.530 & 0.869\\
&  & $\mathcal{L}_{c}, \alpha = 0.3$  & 85.72 & 95.93 & 91.20 & 2.719 & 4.469 & \textbf{0.00423} & 5.470 & 0.866\\
&  & $\mathcal{L}_{c}, \alpha = 0.4$  & 85.65 & 95.95 & \textbf{91.65} & 2.719 & 4.469 & \textbf{0.00423} & 5.670 & 0.869\\
&  & $\mathcal{L}_{c}, \alpha = 0.5$  & 85.28 & 95.76 & 91.22 & \textbf{2.615} & 4.615 & 0.00433 & \textbf{5.320} & \textbf{0.870}\\
 \cline{2-11}
& \multirow{2}{*}{U-Net, 1 ch} & \textit{soft-dice} & 87.46 & 96.35 & 91.18 & 3.094 & 5.042 & 0.00549 & 5.300 & 0.863\\ 
&  & $\mathcal{L}_{c}, \alpha = 0.5$ & \textbf{87.82} & \textbf{96.52} & \textbf{93.03} & \textbf{2.656} & \textbf{4.615} & \textbf{0.00533} & \textbf{4.910} & \textbf{0.872}\\
    \cline{2-11}
& \multirow{6}{*}{U-Net, 2 ch} & \textit{soft-dice} & 87.98 & 96.56 & 90.16 & 2.344 & 4.323 & 0.00507 & 5.550 & 0.855\\ \cdashline{3-11}[2pt/2pt]
&  & $\mathcal{L}_{c}, \alpha = 0.1$ & 88.13 & 96.59 & 91.12 & 2.302 & 4.490 & 0.00465 & 5.180 & \textbf{0.872}\\
&  & $\mathcal{L}_{c}, \alpha = 0.2$ & 87.96 & 96.74 & 92.52 & 2.208 & \textbf{3.979} & 0.00342 & \textbf{4.830} & 0.861\\
&  & $\mathcal{L}_{c}, \alpha = 0.3$ & 87.70 & 96.71 & 92.56 & \textbf{2.115} & 4.521 & \textbf{0.00309} & 5.260 & 0.858\\
&  & $\mathcal{L}_{c}, \alpha = 0.4$ & \textbf{88.57} & \textbf{96.87} & \textbf{93.25} & 2.281 & 4.302 & 0.00327 & 5.370 & 0.868
\\
&  & $\mathcal{L}_{c}, \alpha = 0.5$ & 88.14 & 96.74 & 92.75 & 2.135 & 4.125 & 0.00328 & 5.390 & 0.864\\
\hline\hline

\end{tabular}
\vspace{-1.5em}
\end{table*}

\subsection{Results and Discussion}
We trained two segmentation architectures, a U-Net and an FCN, for the various loss functions in our experimental setup. As a baseline, we trained the networks using \textit{soft-dice} and compared it with the ones trained using the proposed loss (Eq.~\ref{eq3}), by varying $\alpha$ from (0.1 to 0.5).
\vspace{0.15cm}

\noindent\textbf{Quantitative:} We observe that including \textit{soft-clDice} in any  proportion ($\alpha>0$) leads to improved topological, volumetric and graph similarity for all 2D and 3D datasets, see Table \ref{final_table}. We conclude that $\alpha$ can be interpreted as a hyper parameter which can be tuned \emph{per-dataset}. Intuitively, increasing the  $\alpha$ improves the \textit{clDice} measure for most experiments. Most often, \textit{clDice} is high or highest when the graph and topology based measures are high or highest, particularly the $\beta_1$ Error,  Streetmover distance and  Opt-J F1 score; quantitatively indicating that topological properties are indeed represented in the \textit{clDice} measure.

In spite of not optimizing for a high \textit{soft-clDice} on the background class, all of our networks converge to superior segmentation results. This not only reinforces our assumptions on dataset-specific necessary conditions but also validates the practical applicability of our loss. 
Our findings hold for the different network architectures, for 2D or 3D, and for tubular or curvilinear structures, strongly indicating its generalizability to analogous binary segmentation tasks.\\

Observe that CREMI and the synthetic vessel dataset (see Supplementary material) appear to have the smallest increase in scores over the baseline. We attribute this to them being the least complex datasets in the collection, with CREMI having an almost uniform thickness of radii and the synthetic data having a high signal-to-noise ratio and insignificant illumination variation. 
More importantly, we observe larger improvements for all measures in case of the more complex Vessap and Roads data see Figure \ref{road_result}.
In direct comparison to performance measures reported in two recent publications by Hu et al. and Mosinska et al. \cite{hu2019topology,mosinska2018beyond}, we find that our approach is on par or better in terms of Accuracy and Betti Error for the Roads and CREMI dataset. It is important to note that we used a smaller subset of training data for the Road dataset compared to both while using the same test set. 

Hu et al. reported a Betti error for the DRIVE data, which exceeds ours; however, it is important to consider that their approach explicitly minimizes the mismatch of the persistence diagram, which has significantly higher computational complexity during training, see the section below.
We find that our proposed loss performs superior to the baseline in almost every scenario. The improvement appears to be pronounced when evaluating the highly relevant graph and topology based measures, including the recently introduced OPT-Junction F1 by Citraro et al. \cite{citrarotowards}.  Our results are consistent across different network architectures, indicating that \textit{soft-clDice} can be deployed to any network architecture.

\noindent\textbf{Qualitative: }In Figure \ref{road_result}, typical results for our datasets are depicted. Our networks trained on the proposed loss term recover connections, which were false negatives when trained with the soft-Dice loss. These missed connections appear to be particularly frequent in the complex road and DRIVE dataset. For the CREMI dataset, we observe these situations less frequently, which is in line with the very high quantitative scores on the CREMI data. 
Interestingly, in the real 3D vessel dataset, the soft-Dice loss oversegments vessels, leading to false positive connections. This is not the case when using our proposed loss function, which we attribute to its topology-preserving nature. Additional qualitative results can be inspected in the supplementary.\\

\noindent\textbf{Computational Efficiency: }Naturally, inference times of CNNs with the same architecture but different training losses are identical. However, during training, our soft-skeleton algorithm requires $O(kn^2)$ complexity for an $n\times n$ 2D image where $k$ is the number of iterations. As a comparison, \cite{hu2019topology} needs $O(c^2mlog(m))$ (see \cite{han2003topology}) complexity to compute the 1d persistent homology where $d$ is the number of points with zero gradients in the prediction and $m$ is the number of simplices. Roughly, $c$ is proportional to $n^2$, and $m$ is of $O(n^2)$ for a 2D Euclidean grid. Thus, the worst complexity of \cite{hu2019topology} is $O(n^6log(n))$. 
Additionally, their approach requires an $O(clog(c))$ complexity to find an optimal matching of the birth-death pairs. We note that the total run-time overhead for soft-clDice compared to soft-Dice is marginal, i.e., for batch-size of 4 and 1024x1024 image resolution, the former takes 1.35s while the latter takes 1.24s on average ($<$10\% increase) on an RTX-8000.\\

\noindent\textbf{Future Work: }Although our proposed soft-skeleton approximation works well in practice, a better differentiable skeletonization can only improve performance, which we reserve for future research. Any such skeletonization can be readily plugged into our approach. Furthermore, theoretical and experimental multi-class studies would sensibly extend our study.

\section{Conclusive Remarks}
We introduce \textit{clDice}, a novel topology-preserving similarity measure for tubular structure segmentation. Importantly, we present a theoretical guarantee that \textit{clDice} enforces topology preservation up to homotopy equivalence. Next, we use a differentiable version of the \textit{clDice}, \textit{soft-clDice}, in a loss function, to train state-of-the-art 2D and 3D neural networks. We use \textit{clDice} to benchmark segmentation quality from a topology-preserving perspective along with multiple volumetric, topological, and graph-based measures.  We find that training on \textit{soft-clDice} leads to segmentations with more accurate connectivity information, better graph-similarity, better Euler characteristics, and improved Dice and Accuracy. Our \textit{soft-clDice} is computationally efficient and can be readily deployed to any other deep learning-based segmentation tasks such as neuron segmentation in biomedical imaging, crack detection in industrial quality control, or remote sensing.\\

\noindent \textbf{Acknowledgement: } J. C. Paetzold. and S. Shit. are supported by the GCB and Translatum, TU Munich. S.Shit., A. Zhylka. and I. Ezhov. are supported by TRABIT (EU Grant: 765148). We thank Ali Ertuerk, Mihail I. Todorov, Nils Börner and Giles Tetteh.


\begin{figure}[ht!]

\includegraphics[width=0.98\linewidth]{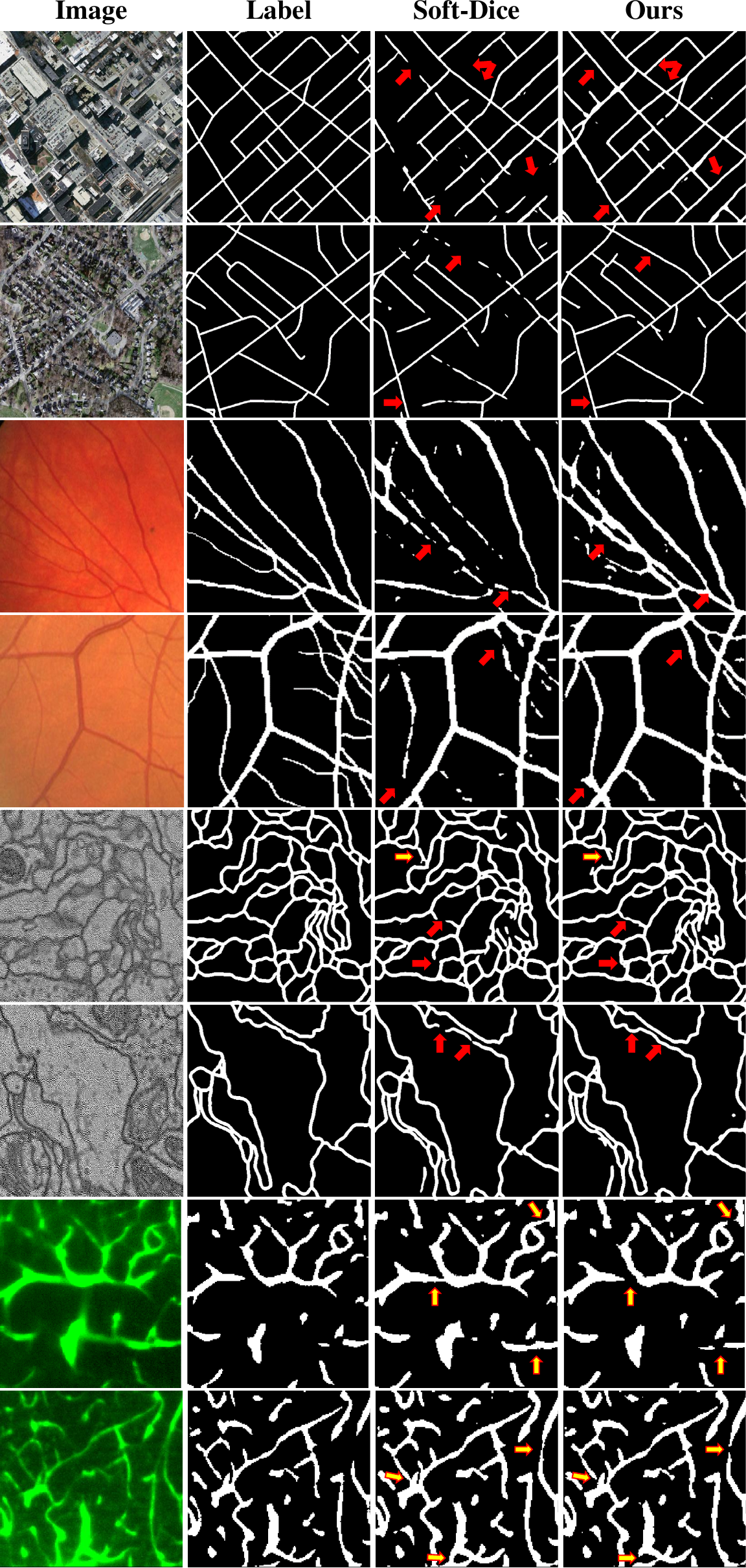}

\caption{Qualitative results: from top to bottom we show two rows of results for: the Massachusetts road dataset, the DRIVE retina dataset, the CREMI neuron data and 2D slices from the 3D Vessap dataset. From left to right, the real image, the label, the prediction using soft-Dice and the U-Net predictions using $\mathcal{L}_c (\alpha=0.5)$ are shown, respectively. The images indicate that \textit{clDice} segments road, retina vessel connections and neuron connections which the soft-Dice loss misses, but also does not segment false-positive vessels in 3D. Some, but not all, missed connections are indicated with solid red arrows, false positives are indicated with red-yellow arrows. More qualitative results can be found in the Supplementary material.}
\label{road_result}
\end{figure}

\clearpage
{\small
\bibliographystyle{ieee_fullname}
\bibliography{main}
}

\clearpage
\appendix
\section{Theory - \textit{clDice} in Digital Topology}
\label{sec:interp}

In addition to our Theorem 1 in the main paper, 
we are providing intuitive interpretations of \textit{clDice} from the digital topology perspective. Betti numbers describe and quantify topological differences in algebraic topology. The first three Betti numbers ($\beta_0$, $\beta_1$, and $\beta_2$) comprehensively capture the manifolds appearing in 2D and 3D topological space. Specifically,
\begin{itemize}[itemsep=-4pt]
    \item $\beta_0$ represents the number of \textit{connected-components},
    \item $\beta_1$ represents the number of \textit{circular holes}
    , and
    \item $\beta_2$ represents the number of \textit{cavities} 
    (Only in 3D)
\end{itemize}{}

\begin{figure}[ht!]
\begin{center}
\includegraphics[width=0.13\textwidth]{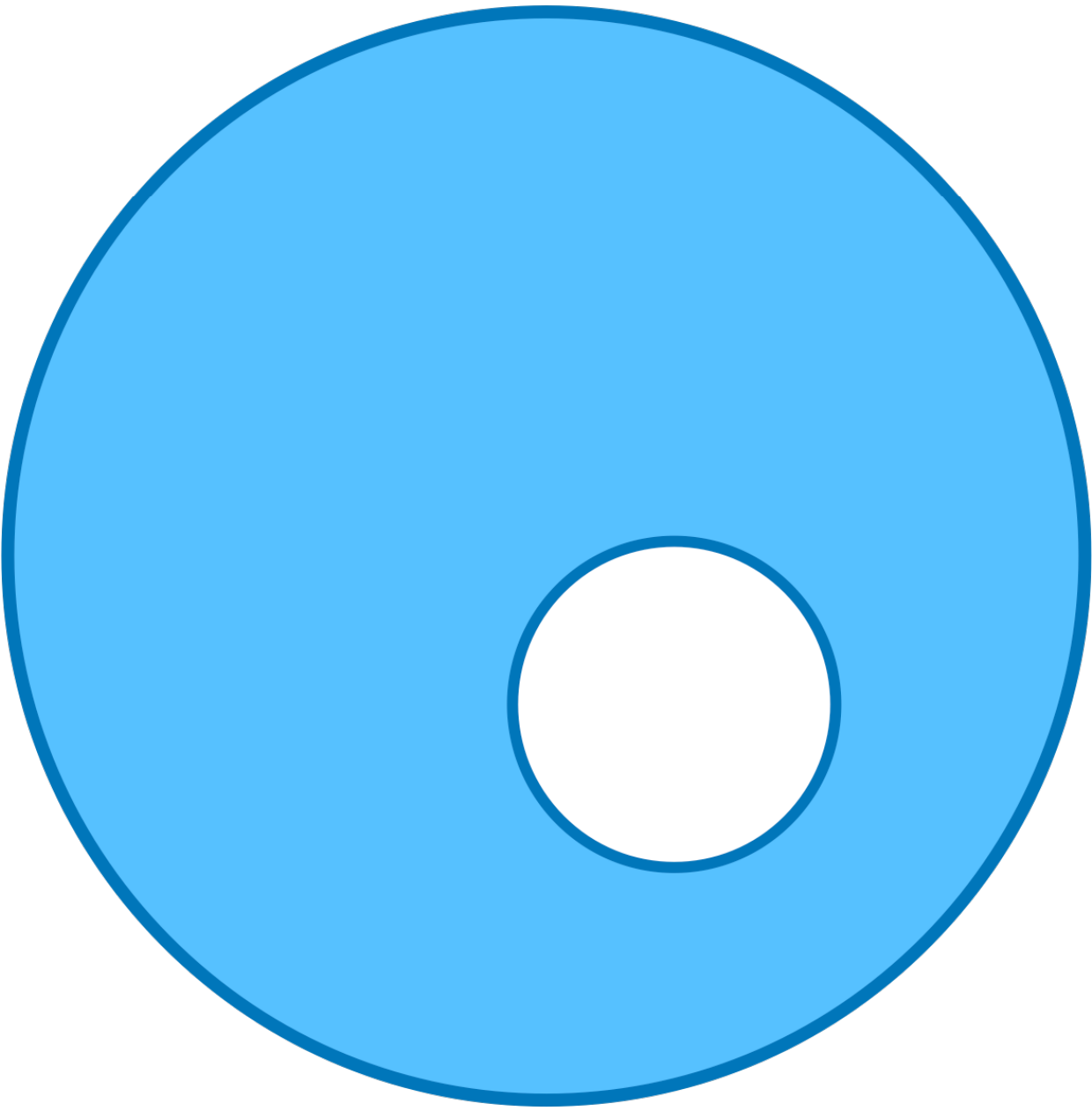}
\includegraphics[width=0.14\textwidth]{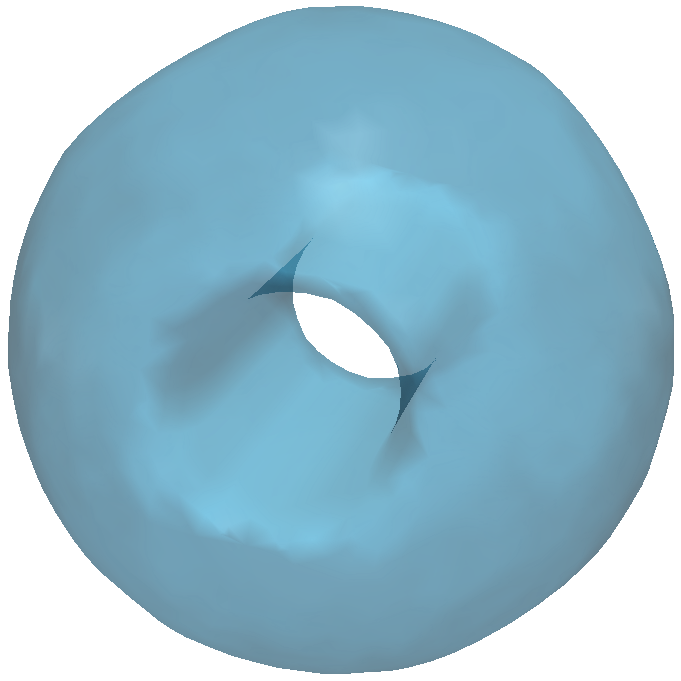}
\includegraphics[width=0.14\textwidth]{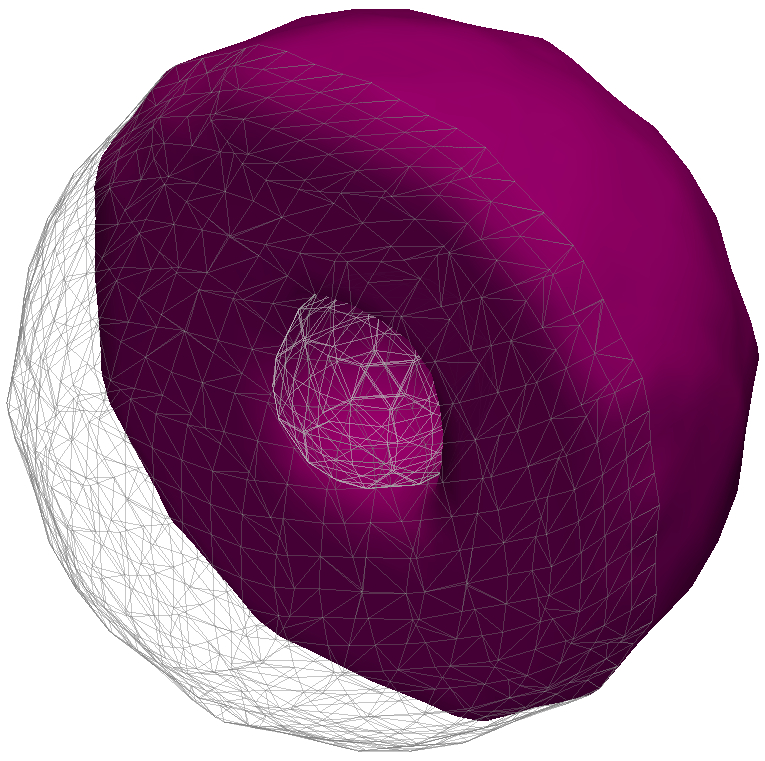} \\
\end{center}
\caption{Examples of the topology properties. Left, a hole in 2D, in the middle a hole in 3D and right a cavity inside a sphere in 3D.} 
\label{fig_top_ex}
\end{figure}

Using the concepts of Betti numbers and digital topology by Kong et al. \cite{kong1995topology,rosenfeld1979digital}, we formulate the effect of topological changes between a true binary mask ($V_L$) and a predicted binary mask ($V_P$) in Fig. \ref{fig_ghost_miss}. We will use the following definition of \textbf{ghosts} and \textbf{misses}, see Figure \ref{fig_ghost_miss}.

\begin{enumerate}[] 
    \item \textbf{Ghosts in skeleton: } We define ghosts in the predicted skeleton ($S_P$) when $S_P \not\subset V_L$. This means the predicted skeleton is not completely included in the true mask. In other words, there exist false-positives in the prediction, which survive after skeletonization.\label{prop1}
    
    \item \textbf{Misses in skeleton: } We define misses in the predicted skeleton ($S_P$) when $S_L \not\subset V_P$. This means the true skeleton is not completely included in the predicted mask. In other words, there are false-negatives in the prediction, which survive after skeletonization.\label{prop2}
\end{enumerate}

\begin{figure*}[ht!]
\begin{center}

\includegraphics[width=0.95\linewidth]{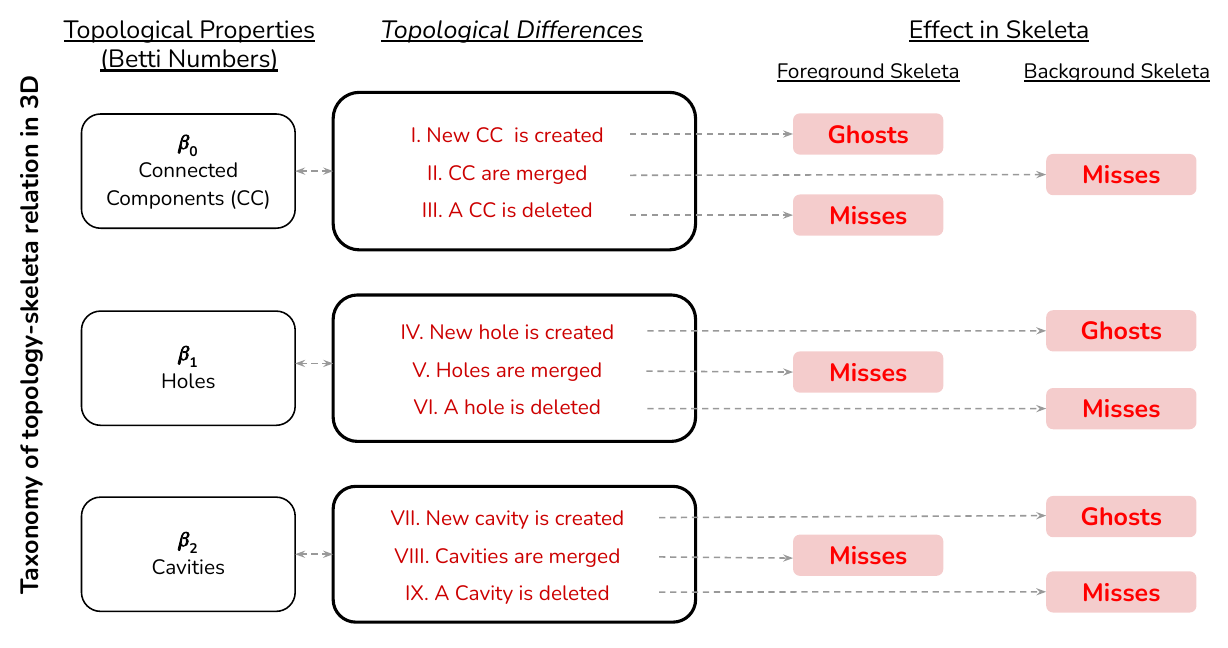}
\label{fig_theory}
\includegraphics[width=0.92\linewidth]{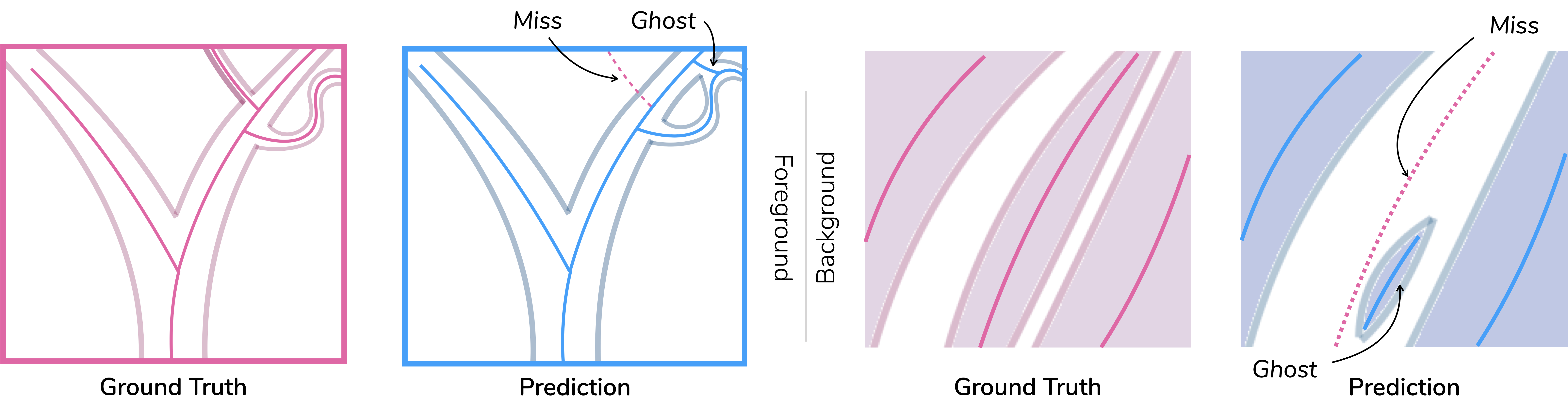}
\end{center}
\caption{Upper part, left, taxonomy of the $iff$ conditions to preserve topology in 3D using the concept of Betti numbers \cite{kong1995topology,kong1989digital}; interpreted as the necessary violation of skeleton properties for any possible topological change in the terminology of ghosts and misses (upper part right) . Lower part, intuitive depictions of ghosts and misses in the prediction; for the skeleton of the foreground (left) and the skeleton of the background (right).}
\label{fig_ghost_miss}
\end{figure*}

The false positives and false negatives are denoted by $V_P\setminus V_L$ and $V_L\setminus V_P$, respectively, where $\setminus$ denotes a set difference operation. The loss function aims to minimize both errors. We call an error correction to happen when the value of a previously false-negative or false-positive voxel flips to a correct value. Commonly used voxel-wise loss functions, such as Dice-loss, treat every false-positive and false-negative equally, irrespective of the improvement in regards to topological differences upon their individual error correction. Thus, they cannot guarantee homotopy equivalence until and unless every single voxel is correctly classified. In stark contrast, we show in the following proposition that \textit{clDice} guarantees homotopy equivalence under \textit{a minimum error correction.}

\begin{proposition}

For any topological differences between $V_P$ and $V_L$, achieving optimal \textit{clDice} to guarantee homotopy equivalence requires a minimum error correction of $V_P$.
\end{proposition}

\begin{proof}
From Fig \ref{fig_ghost_miss}, any topological differences between $V_P$ and $V_L$ will result in ghosts or misses in the foreground or background skeleton. Therefore, removing ghosts and misses are sufficient conditions to remove topological differences. Without the loss of generalizability, we consider the case of ghosts and misses separately:\\

For a \textbf{ghost} $g \subset S_P, \exists \mbox{ a set of predicted voxels } E1 \subset \{V_P \setminus V_L\}\mbox{ such that } V_P \setminus E
1$ does not create any misses and removes $g$. Without the loss of generalizability, let's assume that there is only one ghost $g$. Now, to remove $g$, under a minimum error correction of $V_P$, we have to minimize $|E1|$. Let's say an optimum solution $E1_{min}$ exists. By construction, this implies that $V_P\setminus E1_{min}$ removes $g$.

For a \textbf{miss} $m \subset V_P^\complement, \exists \mbox{ a set of predicted voxels } E2 \subset \{V_L \setminus V_P\}\mbox{ such that } V_P \cup E2$ does not create any ghosts and removes $m$. Without the loss of generalizability, let's assume that there is only one miss $m$. Now, to remove $m$, under a minimum error correction of $V_P$, we have to minimize $|E2|$. Let's say an optimum solution $E2_{min}$ exists. By construction, this implies that $V_P\cup E2_{min}$ removes $m$.\\

Thus, in the absence of any ghosts and misses, from Lemma \ref{obs1}, \textit{clDice}=1 for both foreground and background. Finally, Therefore, Theorem 1 (from the main paper) guarantees homotopy equivalence.
\end{proof}

\begin{restatable}[]{lemma}{obsone}
\label{obs1}
In the absence of any ghosts and misses \textit{clDice}=1.
\end{restatable}

\begin{proof}
The absence of any ghosts $S_P \in V_L$ implies $Tprec=1$; and the absence of any misses $S_L \in V_P$ implies $Tsens=1$. Hence, \textit{clDice}=1.
\end{proof}

\subsection{Interpretation of the Adaption to Highly Unbalanced Data According to Digital Topology:}
Considering the adaptions we described in the main text, the following provides analysis on how these assumptions and adaptions are funded in the concept of ghosts and misses, described in the previous proofs. Importantly, the described adaptions are not detrimental to the performance of \textit{clDice} for our datasets.
We attribute this to the non-applicability of the necessary conditions specific to the background (i.e. II, IV, VI, VII, and IX in Figure \ref{fig_theory}), as explained below: 

\begin{itemize}[,itemsep=-2pt]
    \item II. $\rightarrow$ In tubular structures, all foreground objects are eccentric (or anisotropic). Therefore isotropic skeletonization will highly likely produce a ghost in the foreground.
    \item IV. $\rightarrow$ Creating a hole outside the labeled mask means adding a ghost in the foreground. Creating a hole inside the labeled mask is extremely unlikely because no such holes exist in our training data.
    \item VI. $\rightarrow$ The deletion of a hole without creating a miss is extremely unlikely because of the sparsity of the data.
    \item VII.and IX. (only for 3D) $\rightarrow$ Creating or removing a cavity is very unlikely because no cavities exist in our training data.
\end{itemize}

\section{Additional Qualitative Results} 
\begin{figure*}[]
\label{supp_road}
\centering

\includegraphics[width=0.78\linewidth]{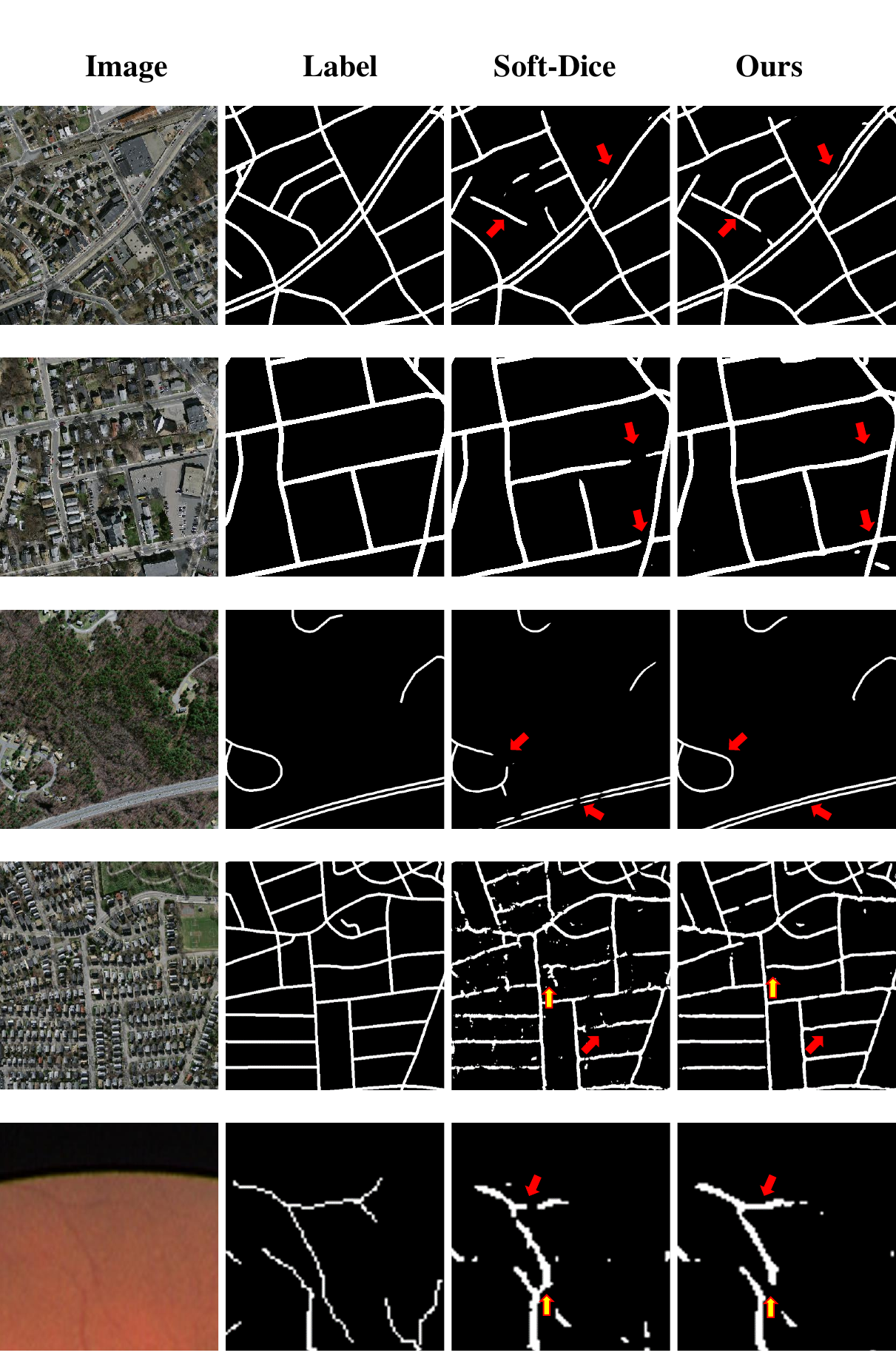}

\caption{\footnotesize Qualitative results:  for the Massachusetts Road dataset and  for the DRIVE retina dataset (last row). From left to right, the real image, the label, the prediction using soft-dice and the predictions using the proposed $\mathcal{L}_c (\alpha=0.5)$, respectively. The first three rows are U-Net results and the fourth row is an FCN result. This indicates that \textit{soft-clDice} segments road connections which the soft-dice loss misses. Some,  but  not  all,  missed  connections  are  indicated with solid red arrows, false positives are indicated with red-yellow arrows.}
\end{figure*}

\begin{figure*}[]
\label{supp_3d}
\centering

\includegraphics[width=0.78\linewidth]{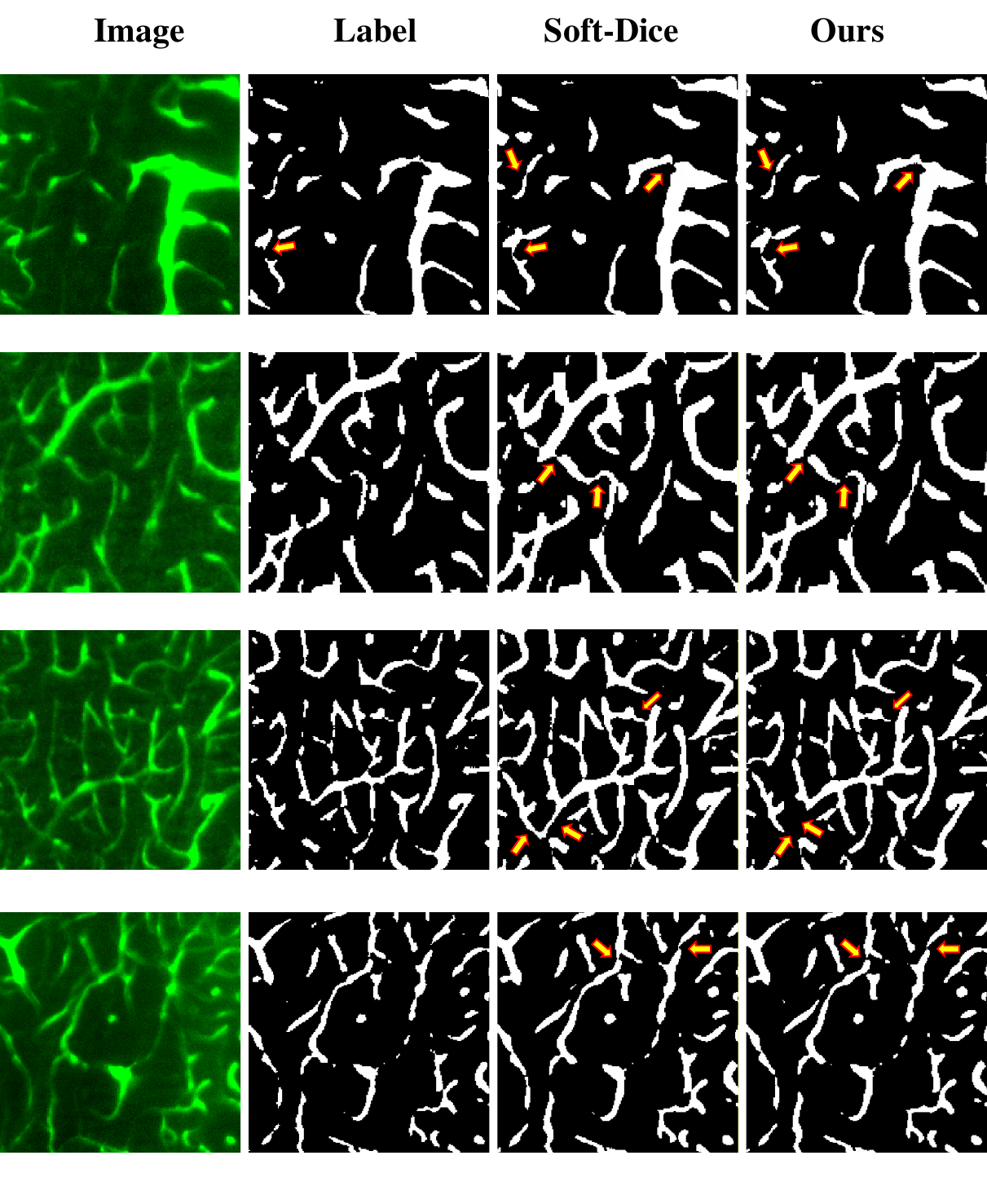}
\caption{\footnotesize Qualitative results: 2D slices of the 3D vessel dataset for different sized field of views. From left to right, the real image, the label, the prediction using soft-dice and the U-Net predictions using $\mathcal{L}_c (\alpha=0.4)$, respectively. These images show that \textit{soft-clDice} helps to better segment the vessel connections. Importantly the networks trained using soft-dice over-segment the vessel radius and segments incorrect connections. Both of these errors are not present when we train including \textit{soft-clDice} in the loss. Some,  but  not  all, false positive connections are indicated with red-yellow arrows.
}
\end{figure*}

\section{Comparison to Other Literature:} 

A recent pre-print proposed a region-separation approach, which aims to tackle the issue by analysing disconnected foreground elements \cite{oner2020promoting}. Starting with the predicted distance map, a network learns to close ambiguous gaps by referring to a ground truth map which is dilated by a five-pixel kernel, which is used to cover the ambiguity. However, this does not generalize to scenarios with a close or highly varying proximity of the foreground elements (as is the case for e.g. capillary vessels, synaptic gaps or irregular road intersections). Any two foreground objects which are placed at a twice-of-kernel-size distance or closer to each other will potentially be connected by the trained network. This is facilitated by the loss function considering the gap as a foreground due to performing dilation in the training stage. Generalizing their approach to smaller kernels has been described as infeasible in their paper \cite{oner2020promoting}.

\section{Datasets and Training Routine}

For the DRIVE vessel segmentation dataset, we perform three-fold cross-validation with 30 images and deploy the best performing model on the test set with 10 images. For the Massachusetts Roads dataset, we choose a subset of  120 images (ignoring imaged without a network of roads) for three-fold cross-validation and test the models on the 13 official test images.  For CREMI, we perform three-fold cross-validation on 324 images and test on 51 images. For the 3D synthetic dataset. we perform experiments using 15 volumes for training, 2 for validation, and 5 for testing.
For the Vessap dataset, we use 11 volumes for training, 2 for validation and 4 for testing. In each of these cases, we report the performance of the model with the highest clDice score on the validation set.

\section{Network Architectures} 
We use the following notation: $In(input~channels)$, $Out(output~channels)$, \\ $B(output~channels)$ present input, output, and bottleneck information(for U-Net); $C(filter~size, output~channels)$ denote a convolutional layer followed by $ReLU$ and batch-normalization; $U(filter~size, output~channels)$ denote a trans-posed convolutional layer followed by $ReLU$ and batch-normalization; $\downarrow 2$ denotes maxpooling; $\oplus$ indicates concatenation of information from an encoder block. We had to choose a different FCN architecture for the Massachusetts road dataset because we realize that a larger model is needed to learn useful features for this complex task.

\subsection{Drive Dataset}
\subsubsection{FCN :}
$IN(\mbox{3 ch})\rightarrow C(3,5)\rightarrow C(5,10) \rightarrow C(5,20)\rightarrow C(3,50)\rightarrow C(1,1)\rightarrow Out(1)$
\subsubsection{Unet :}
\paragraph{ConvBlock :} $C_B(3, out~size)\equiv C(3,out~size)\rightarrow C(3,out~size)\rightarrow \downarrow 2$
\paragraph{UpConvBlock:} $U_B(3, out~size)\equiv U(3,out~size)\rightarrow \oplus\rightarrow C(3,out~size)$
\paragraph{Encoder :}
$IN(\mbox{3 ch})\rightarrow C_B(3,64)\rightarrow C_B(3,128) \rightarrow C_B(3,256)\rightarrow C_B(3,512)\rightarrow C_B(3,1024)\rightarrow B(1024)$
\paragraph{Decoder :}
$B(1024)\rightarrow U_B(3,1024)\rightarrow U_B(3,512)\rightarrow U_B(3,256) \rightarrow U_B(3,128)\rightarrow U_B(3,64)\rightarrow Out(1)$

\subsection{Road Dataset}
\subsubsection{FCN :}
$IN(\mbox{3 ch})\rightarrow C(3,10)\rightarrow C(5,20)\rightarrow C(7,30)\rightarrow C(11,30)\rightarrow C(7,40) \rightarrow C(5,50)\rightarrow C(3,60)\rightarrow C(1,1)\rightarrow Out(1)$
\subsubsection{Unet :}
Same as Drive Dataset, except we used 2x2 up-convolutions instead of bilinear up-sampling followed by a 2D-convolution with kernel size 1.

\subsection{Cremi Dataset}
\subsubsection{Unet :}
Same as Road Dataset. 

\subsection{3D Dataset }
\subsubsection{3D FCN :}
$IN(\mbox{1 or 2 ch})\rightarrow C(3,5)\rightarrow C(5,10) \rightarrow C(5,20)\rightarrow C(3,50)\rightarrow C(1,1)\rightarrow Out(1)$
\subsubsection{3D Unet :}
\paragraph{ConvBlock :} $C_B(3, out~size)\equiv C(3,out~size)\rightarrow C(3,out~size)\rightarrow \downarrow 2$
\paragraph{UpConvBlock:} $U_B(3, out~size)\equiv U(3,out~size)\rightarrow \oplus\rightarrow C(3,out~size)$
\paragraph{Encoder :}
$IN(\mbox{1 or 2 ch})\rightarrow C_B(3,32)\rightarrow C_B(3,64) \rightarrow C_B(3,128)\rightarrow C_B(5,256)\rightarrow C_B(5,512)\rightarrow B(512)$
\paragraph{Decoder :}
$B(512)\rightarrow U_B(3,512) \rightarrow U_B(3,256)\rightarrow U_B(3,128) \rightarrow U_B(3,64)\rightarrow U_B(3,32)\rightarrow Out(1)$

\begin{table}[ht!]
\centering
\caption{Total number of parameters for each of the architectures used in our experiment.}
\begin{tabular}{ccc}
\hline
Dataset & Network & Number of parameters \\
\hline
\hline
Drive   & FCN     & 15.52K               \\
        & UNet    & 28.94M               \\
        \hline
Road    & FCN      & 279.67K              \\
\hline
Cremi    & UNet      & 31.03M              \\
\hline
3D      & FCN 2ch    & 58.66K               \\
        & Unet 2ch    & 19.21M             \\
\hline
\end{tabular}
\end{table}

\section{Soft Skeletonization Algorithm}

\begin{figure}[ht!]
\centering
\includegraphics[width=1\linewidth]{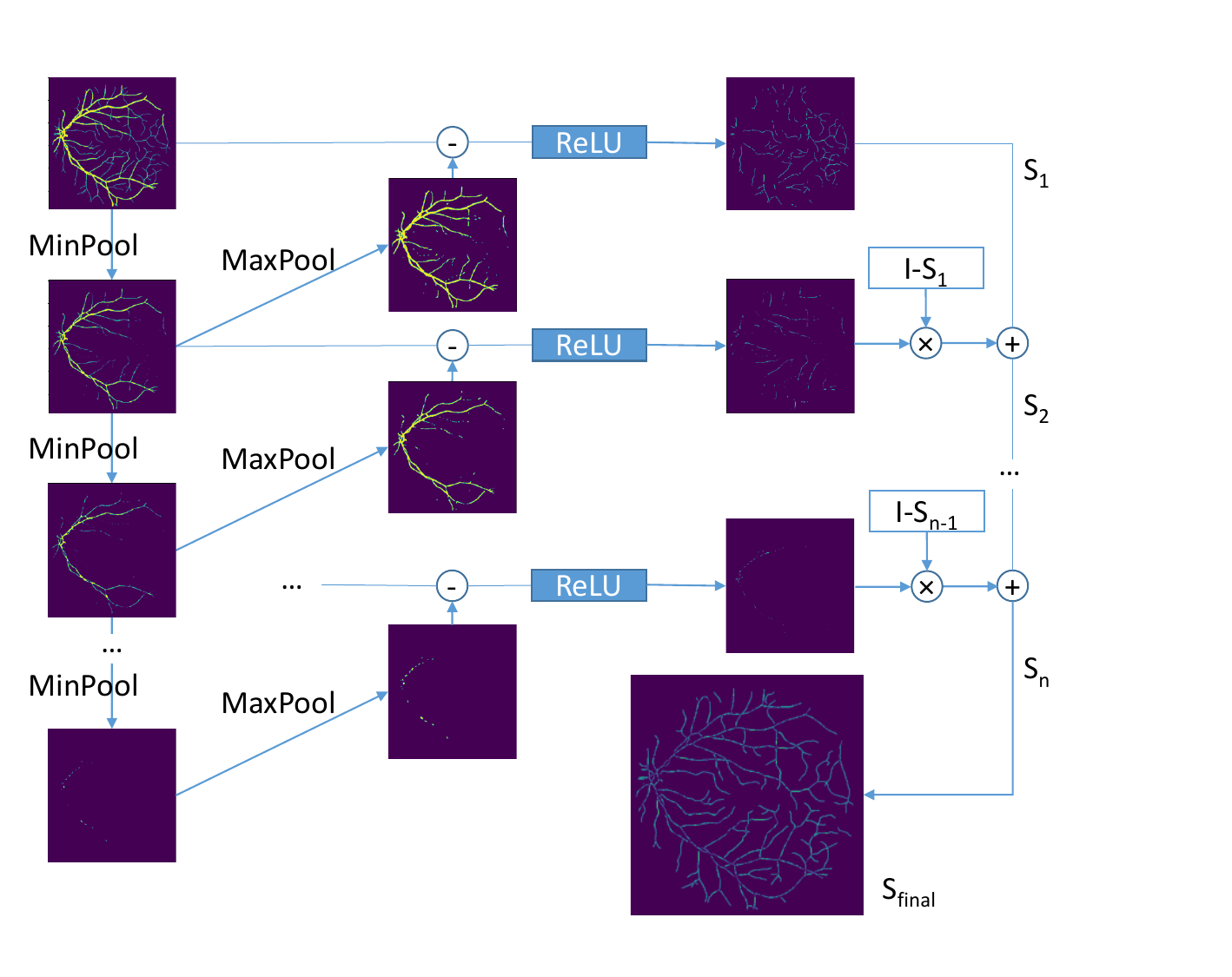}

\caption{Scheme of our proposed differentiable skeletonization. On the top left the mask input is fed. Next, the input is reatedly eroded and dilated. The resulting erosions and dilations are compared to the image before dilation. The difference between thise images is part of the skeleton and will be added iteratively to obtain a full skeletonization. The ReLu operation eliminates pixels that were generated by the dilation but are not part of the oirginal or eroded image.}
\label{example_skel}
\end{figure}

\section{Code for the \textit{clDice} similarity measure and the \textit{soft-clDice} loss (PyTorch):}

\subsection{\textit{clDice} measure}
\footnotesize
\begin{lstlisting}[language=Python]
from skimage.morphology import skeletonize
import numpy as np
def cl_score(v, s):
    return np.sum(v*s)/np.sum(s)
def clDice(v_p, v_l):
    tprec = cl_score(v_p,skeletonize(v_l))
    tsens = cl_score(v_l,skeletonize(v_p))
    return 2*tprec*tsens/(tprec+tsens)
\end{lstlisting}

\subsection{\textit{soft-skeletonization} in 2D}
\begin{lstlisting}[language=Python]
import torch.nn.functional as F
def soft_erode(img):
    p1 = -F.max_pool2d(-img, (3,1), (1,1), (1,0))
    p2 = -F.max_pool2d(-img, (1,3), (1,1), (0,1))
    return torch.min(p1,p2)

def soft_dilate(img):
    return F.max_pool2d(img, (3,3), (1,1), (1,1))

def soft_open(img):
    return soft_dilate(soft_erode(img))
    
def soft_skel(img, iter):
    img1 = soft_open(img)
    skel = F.relu(img-img1)
    for j in range(iter):
        img = soft_erode(img)
        img1 = soft_open(img)
        delta = F.relu(img-img1)
        skel = skel + F.relu(delta-skel*delta)
    return skel
\end{lstlisting}

\subsection{\textit{soft-skeletonization} in 3D}
\begin{lstlisting}[language=Python]
import torch.nn.functional as F

def soft_erode(img):
    p1 = -F.max_pool3d(-img,(3,1,1),(1,1,1),(1,0,0))
    p2 = -F.max_pool3d(-img,(1,3,1),(1,1,1),(0,1,0))
    p3 = -F.max_pool3d(-img,(1,1,3),(1,1,1),(0,0,1))

    return torch.min(torch.min(p1, p2), p3)

def soft_dilate(img):
    return F.max_pool3d(img,(3,3,3),(1,1,1),(1,1,1))

def soft_open(img):
    return soft_dilate(soft_erode(img))

def soft_skel(img, iter_):
    img1  =  soft_open(img)
    skel  =  F.relu(img-img1)
    for j in range(iter_):
        img  =  soft_erode(img)
        img1  =  soft_open(img)
        delta  =  F.relu(img-img1)
        skel  =  skel +  F.relu(delta-skel*delta)
    return skel

\end{lstlisting}
\normalsize

\section{Evaluation Metrics}
\noindent As discused in the text, we compare the performance of various experimental setups using three types of metrics: volumetric, graph-based and topology-based.

\subsection{ Overlap-based:}
Dice coefficient, Accuracy and \textit{clDice}, we calculate these scores on the whole 2D/3D volumes. \textit{clDice} is calculated using a morphological skeleton (skeletonize3D from the skimage library). 

\subsection{Graph-based:} We extract graphs 
from random patches of $64\times64$ pixels in 2D and $48\times48\times48$ in 3D images.  

For the StreetmoverDistance (SMD) \cite{belli2019image} we uniformly sample a fixed number of points from the graph of the prediction and label, match them and calculate the Wasserstein-distance between these graphs. 
For the junction-based metric (Opt-J) we compute the F1 score of junction-based metrics, recently proposed by \cite{citrarotowards}. According to their paper this metric is advantageous over all previous junction-based metrics as it can account for nodes with an arbitrary number of incident edges, making this metric more sensitive to endpoints and missed connections in predicted networks. For more information please refor to their paper.

\subsection{Topology-based:} 
For topology-based scores we calculate the Betti Errors for the Betti Numbers $\beta_0$ and $\beta_1$. Also, we calculate the Euler characteristic, $\chi = V-E+F$, where $E$ is the number of edges, $F$ is the number of faces and $V$ is the number of vertices. We report the relative Euler characteristic error ($\chi_{ratio}$), as the ratio of the $\chi$ of the predicted mask and that of the ground truth. Note that a $\chi_{ratio}$ closer to one is preferred. All three topology-based scores are calculated on random patches of $64\times64$ pixels in 2D and $48\times48\times48$ in 3D images.

\section{Additional Quantitative Results}

\begin{table}[ht!]

\caption{ Quantitative experimental results for the 3D synthetic vessel dataset. Bold numbers indicate the best performance. We trained baseline models of binary-cross-entropy (BCE), softDice and mean-squared-error loss (MSE) and combined them with our \textit{soft-clDice} and varied the $\alpha > 0$. For all experiments we observe that using \textit{soft-clDice} in $\mathcal{L}_{c}$ results in improved scores  compared to \emph{soft-Dice}. This improvement holds for almost $\alpha > 0$. We observe that \textit{soft-clDice} can be efficiently combined with all three frequently used loss functions.}

\centering
\label{synth_data_table}

\begin{tabular}{|p{2.2cm}|p{1.4cm}|p{1.4cm}|}
        \hline
        Loss&Dice&clDice\\
        \hline
        BCE&99.81&98.24\\
        \hdashline
        $L_c$, $\alpha$ = 0.5&99.76&98.25\\
        $L_c$, $\alpha$ = 0.4&99.77&98.29\\
        $L_c$, $\alpha$ = 0.3&99.76&98.20\\
        $L_c$, $\alpha$ = 0.2&99.78&98.29\\
        $L_c$, $\alpha$ = 0.1&99.82&98.39\\
        $L_c$, $\alpha$ = 0.01&99.83&98.46\\
        $L_c$, $\alpha$ = 0.001&\textbf{99.85}&\textbf{98.42}\\
        \hline
        soft-Dice&99.74&97.07\\
  		\hdashline
        $L_c$, $\alpha$ = 0.5&99.74&97.53\\
        $L_c$, $\alpha$ = 0.4&99.74&97.07\\
        $L_c$, $\alpha$ = 0.3&\textbf{99.80}&\textbf{98.13}\\
        $L_c$, $\alpha$ = 0.2&99.74&97.08\\
        $L_c$, $\alpha$ = 0.1&99.74&97.08\\
        $L_c$, $\alpha$ = 0.01&99.74&97.07\\
        $L_c$, $\alpha$ = 0.001&99.74&97.12\\
        \hline
        MSE&99.71&97.03\\
  		\hdashline
        $L_{c}$, $\alpha$ = 0.5&99.62&98.22\\
        $L_c$, $\alpha$ = 0.4&99.65&97.04\\
        $L_c$, $\alpha$ = 0.3&99.67&98.16\\
        $L_c$, $\alpha$ = 0.2&99.70&97.10\\
        $L_c$, $\alpha$ = 0.1&99.74&98.21\\
        $L_c$, $\alpha$ = 0.01&99.82&98.32\\
        $L_c$, $\alpha$ = 0.001&\textbf{99.84}&\textbf{98.37}\\
   \hline \hline
\end{tabular}

\end{table}

\end{document}